%% file: main.tex
\documentclass[twocolumn]{autart}    

\usepackage{graphicx}          

\usepackage{xspace}
\usepackage{times}
\usepackage{soul}
\usepackage{url}
\usepackage[hidelinks]{hyperref}
\usepackage[utf8]{inputenc}
 \usepackage{graphicx}
\usepackage{booktabs}

\usepackage{tabularx}
\usepackage[english]{babel}

\usepackage{amsmath,amssymb}

\usepackage{mathtools}
\usepackage{subcaption}

\usepackage{tikz}
\usetikzlibrary{automata, arrows.meta, positioning}
\usepackage{acronym}
\usepackage{changes}

\usepackage{algorithm}
\usepackage[noend]{algpseudocode}

\definecolor{orange}{rgb}{1,0.5,0}
\definecolor{darkgreen}{rgb}{0,204,0}

\usetikzlibrary{shapes,shapes.geometric,arrows,fit,calc,positioning,automata,}

\input{defs}

\usepackage{enumerate}

\begin{document}
\tikzset{estate/.style={draw,ellipse,minimum width=1cm,minimum height=0.5cm}}

\begin{frontmatter}

\title{  Automata Learning  of Preferences  over Temporal Logic Formulas from Pairwise Comparisons\thanksref{footnoteinfo}} 

\thanks[footnoteinfo]{This paper was not presented at any IFAC 
meeting. Corresponding author J. FU. }

\author[Hazhar]{Hazhar Rahmani}\ead{hrahmani@missouristate.edu},    
\author[Jie]{Jie Fu}\ead{fujie@ufl.edu}               

\address[Hazhar]{Department of Computer Science, Missouri State University, USA}  
\address[Jie]{Department Electrical and Computer Engineering , University of Florida, USA}             


\begin{abstract}                          
Many preference elicitation algorithms consider preference over propositional logic formulas or items with different attributes. In sequential decision making, a user's preference can be a preorder over possible outcomes, each of which is a temporal sequence of events. This paper considers a class of preference inference problems where the user's unknown preference is represented by a preorder over regular languages (sets of temporal sequences), referred to as temporal goals. Given a finite set of pairwise comparisons between finite words, the objective is to learn both the set of temporal goals and the preorder over these goals. We first show that a preference relation over temporal goals can be modeled by a Preference Deterministic Finite Automaton (PDFA), which is a deterministic finite automaton augmented with a preorder over acceptance conditions. The problem of preference inference reduces to learning the PDFA. This problem is shown to be computationally challenging, with the problem of determining whether there exists a PDFA  of size smaller than a given integer $k$, consistent with the sample, being NP-Complete. We formalize the properties of characteristic samples   and develop an algorithm that guarantees to learn, given a characteristic sample, the minimal PDFA equivalent to the true PDFA from which the sample is drawn. We present the method through a running example and provide detailed analysis using a robotic motion planning problem.
\end{abstract}

\end{frontmatter}

\section{Introduction}
  \label{sec:intro}


%
%
\begin{figure}[h!]
		\centering
		\begin{subfigure}[b]{0.5\linewidth}
			\centering
			\includegraphics[scale=0.44]{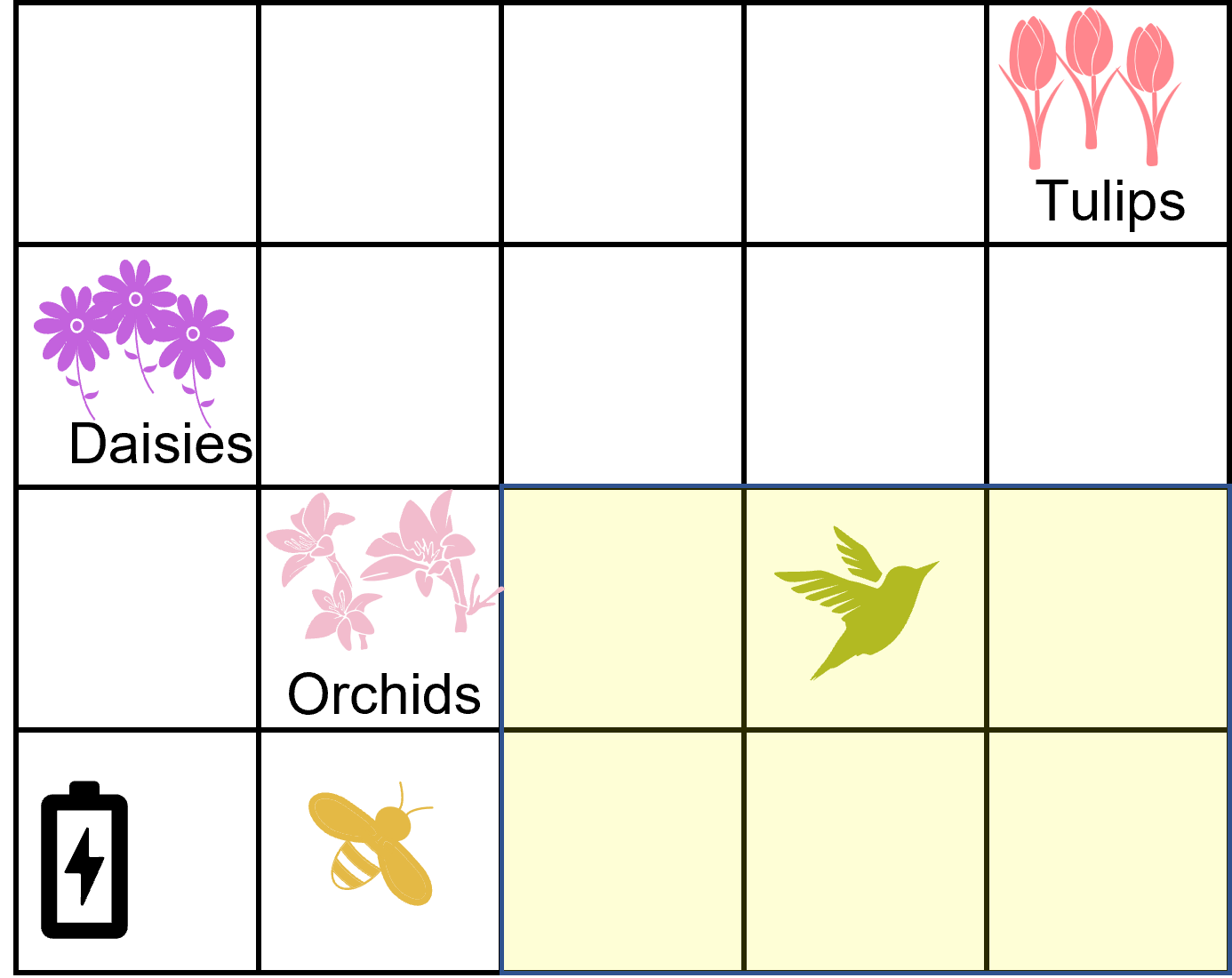}
   \caption{Garden and bee robot.}
   \label{fig:garden}
		\end{subfigure}
		\hfill
  \centering
		\begin{subfigure}[b]{0.48\linewidth}
			\centering
			\includegraphics[scale=0.44]{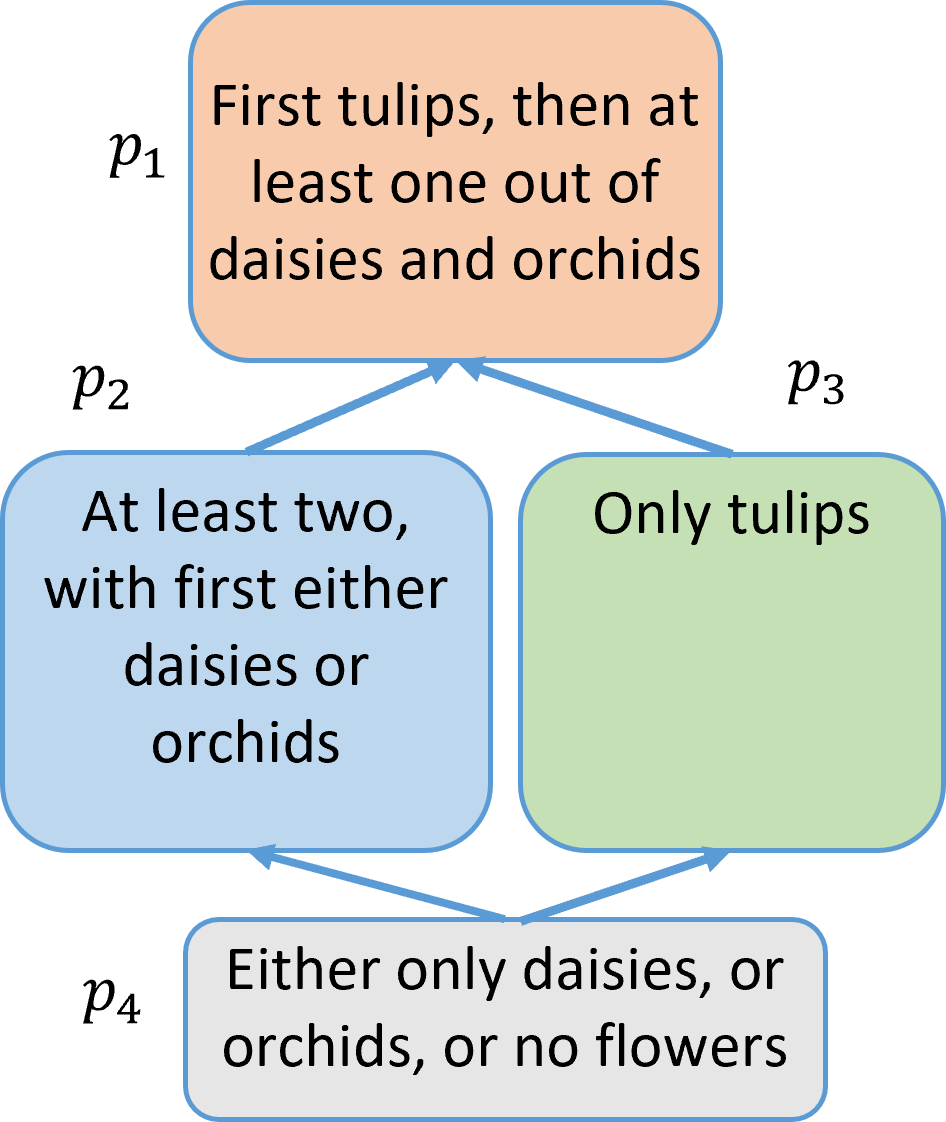}
   \caption{User preferences.}
     \label{fig:garden_preferences}
		\end{subfigure}
		\hfill
  \centering
		\begin{subfigure}[b]{0.95\linewidth}
			\centering
			\includegraphics[scale=0.54]{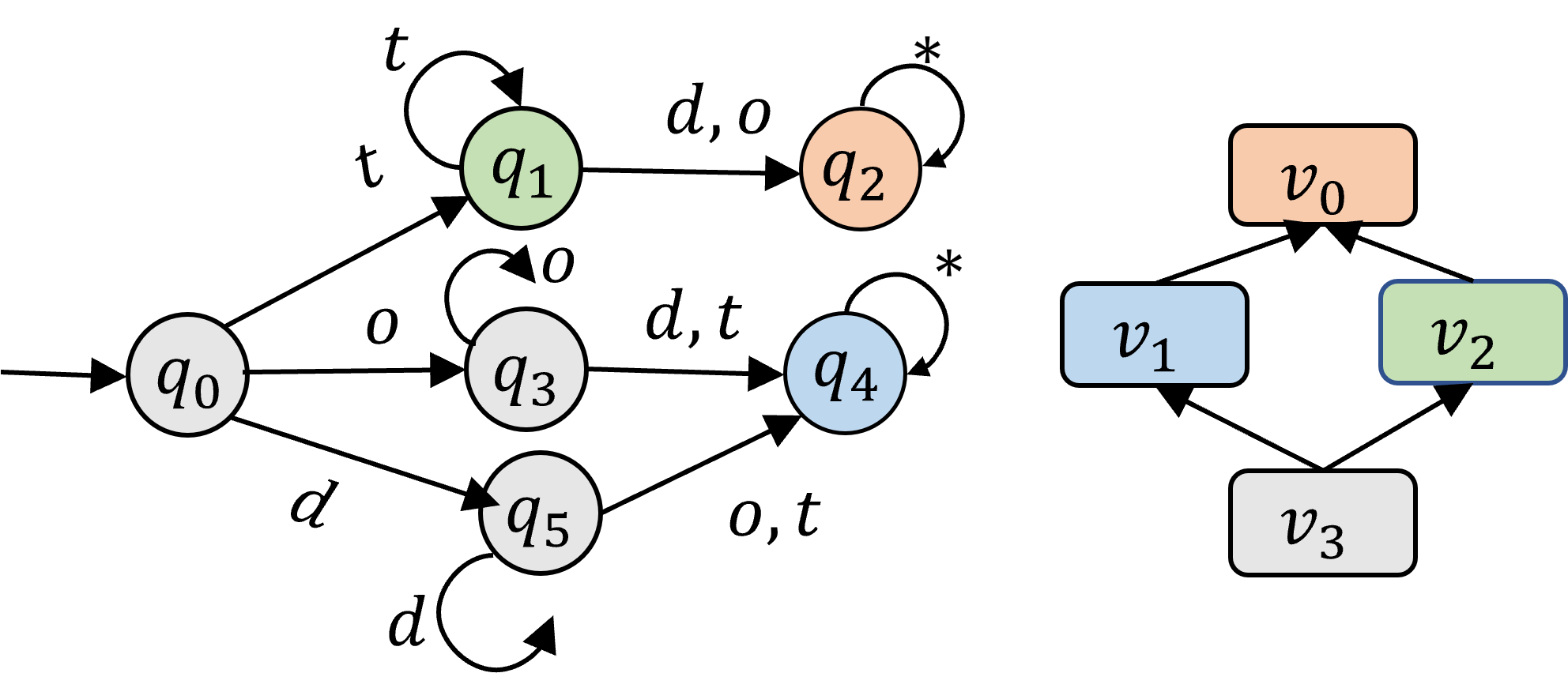}
   \caption{\ac{pdfa} encoding the preferences. For each state $q\in Q$, $\delta(q,\varnothing)=q$ and the selfloops with $\varnothing$ are omitted.}
     \label{fig:garden_pdfa}
		\end{subfigure}

        \caption{A Garden example in which the robot is tasked with pollinating the flowers of the garden.}
  \label{fig:garden_example}
	\end{figure}

As robotics and artificial intelligence (AI) continue to play an increasingly prominent role in society, the integration of user preferences into autonomous decision-making is becoming essential. This integration not only enhances human-AI alignment but also enables more personalized and effective decision support. The study of preference representation, planning, and learning has long been an active area of research in AI (see comprehensive surveys  in \cite{pigozzi2016preferences,jorge2008planning}).
%
%
%

In recent years, the focus of the community expanded from classical planning problems with simple goal-reaching objectives to a richer class of planning problems with temporally extended goals, specified in temporal logic.  Simultaneously, research on preference-based planning (PBP) methods has evolved to address user preferences specifically over these temporal goals.

Several studies~\cite{tumova2013least,wongpiromsarn2021,rahmani2020what,rahmani2019optimal} have explored preference-based planning (PBP) in the context of \emph{minimum violation} planning problems, which aim to compute policies that prioritize higher-priority goals while minimally violating lower-priority ones in the presence of conflicting temporal goals.
Other works~\cite{mehdipourSpecifyingUserPreferences2021,cardona2023mixed,li2023probabilistic} have proposed new formal specifications to succinctly represent preferences over temporal goals and developed algorithms for synthesizing optimal plans   respecting these preferences. While most PBP methods assume that the user's preference induces a total (complete) order, where all outcomes are pairwise comparable, recent research~\cite{fu2021probabilistic,kulkarni2022opportunistic} has addressed incomplete preferences, allowing  some goals to be incomparable. 
In recent work ~\cite{rahmani2023probabilistic}, we introduced a novel automaton variant, called the preference deterministic finite automaton, which can be constructed from a preorder over temporal goals specified in syntactically co-safe linear temporal logic. The procedure for constructing a preference automaton is detailed in \cite{rahmani2023probabilistic}\footnote{An online tool is available at \textcolor{blue}{\url{https://akulkarni.me/prefltlf2pdfa.html}}}. This preference automaton offers a compact representation of the user's preorder over temporal goals. Furthermore, we developed probabilistic planning algorithms to compute the most preferred, non-comparable policies, based on the given preference specification.

 In this work, we consider the problem of   learning preferences over temporally   goals from human feedback. 
As a motivation,  consider a robotic motion planning problem (Figure~\ref{fig:garden}, introduced in~\cite{rahmani2023probabilistic}) with a human supervisor in the loop:
A bee robot is tasked with pollinating a garden with three types of flowers---tulips, daisies, and orchids.
The robot operates under a limited battery, and faces uncertain weather and the presence of a bird.
%
The robot cannot pollinate a flower when it is raining, and when it meets the bird, it has to stay put until the bird goes away.
The human's preference on the possible robot behaviors is modeled by Figure~\ref{fig:garden_preferences}, which is a preorder over four temporal goals:
    \begin{itemize}
        \item[($p_1$)] pollinating tulips first, followed by at least one other type of flower;
        \item[($p_2$)] pollinate two types of flowers, starting with either daisies or orchids; 
        \item[($p_3$)] pollinate only tulips; and
        \item[($p_4$)] either pollinate no flower or pollinate one type between daisies and orchids.
    \end{itemize}
Among these temporal goals,  $p_1$ is the most preferred and $p_4$ is the least preferred outcome, while $p_2$ and $p_3$ are incomparable with each other. The human's preference is encoded using the preference DFA in Figure~\ref{fig:garden_pdfa}. The alphabet is $\Sigma=\{t, d, o\}$, with $t$, $d$, and $o$ representing pollinating tulips, daisies, and orchids, respectively. The emptyset is denoted $\varnothing$.


  Assume the robot has no prior knowledge about human preference but can ask the supervisor to compare a set of trajectories, each of which translates to a temporal sequence: For example, $\varnothing \varnothing t \varnothing d$ means that the tulip was pollinated first and then   daisy is pollinated next. The robot asks human simple questions, is $\varnothing \varnothing t \varnothing d$ preferred to $\varnothing d \varnothing o$? It receives one of three answers (yes, no, or incomparable) from human.  
 %
The goal of the robot, is to learn the preference DFA from those comparisons: Not only the robot  should learn what are the temporal goals but also what   the preorder on them is. We ask the following problems: (1) What is the hardness of the learning problem?   (2)   Under what conditions it is possible to learn the correct preference, or an equivalent PDFA generated by ground-truth preference?

 \paragraph*{Related work}
 
 Temporal logic is often employed to describe temporal behavior of a system or time series data. Motivated by the need to construct a succinct specification of system behaviors,  learning temporal logic formulas from a sample, consisting of    positive and negative
examples,  
 has been   extensively studied, for Linear Temporal Logic (LTL)~\cite{neider2018learning,ghiorzi2023learning,raha2022scalable,fijalkow2021complexity,camacho2019learning,kim2019bayesian,gaglione2021learning,ielo2023towards,arif2020syslite}, Computational Tree Logic (CTL)~\cite{roy2023inferring,pommellet2024sat}, Metric Temporal Logic (MTL)\cite{raha2023synthesizing}, and Signal Temporal Logic (STL)\cite{nenzi2018robust,mohammadinejad2020interpretable}.
These learning methods use a variety of techniques, including SAT-based methods, integer linear programming, and the search for optimal parameters using formula templates. 
Shah \etal~\cite{shah2023learning} proposes an active approach of learning formal specifications that uses membership and comparison operations, where the first operation tells whether a word belongs to the specification language or not, and the second operation tells which one of two chosen words is preferred to the other. The approach assumes the oracle produces a preference relation that is \emph{membership-respecting}, i.e., words outside the language cannot be preferred to words within the language. Our learning problem differs from traditional temporal logic inference in that, while the goal of temporal logic inference is typically to infer a single formula, our objective is to learn an automaton that represents the user's preference over a set of temporal logic formulas, which are also unknown.

%
Learning user preferences through human feedback employs active preference learning with online planning~\cite{wilde2018learning,wilde2020active,Sadigh2017ActivePL,akrour2012april}. In this approach, the autonomous system iteratively presents the user with different alternatives, and after receiving feedback, it learns about the user's preferences and improves its plan based on the current preference model.
Another closely related line of work is reinforcement learning from human feedback (RLHF)(see a recent survey \cite{kaufmann2023survey}), which has attracted significant research interests due to its application in human-AI-alignment. RLHF includes two classes of methods: One class of methods directly learns a reward function from pairwise comparisons, and another class of methods directly learns a policy from the feedback. Most RLHF methods consider the  Bradley-Terry-Luce (BTL) preference model, which assumes the preference is a total order and can be captured by a reward function. These methods  do not extend to learning a preorder over temporally evolving outcomes.

\emph{Another related class of methods is preference-based reinforcement learning} (PbRL)   (see a survey \cite{wirth2017survey}), which aim to solve the \emph{reinforcement learning} (RL) in situations where it is challenging to design a numerical reward function.
The robot learns the policy via learning the user preference over states, actions, or trajectories.
Overall, these methods make RL to be more accessible to non-expert users.  In comparison to RLHF and PbRL, the proposed method focuses on learning preferences over temporal goals, each of which can be satisfied by a subset of trajectories.

The contributions in this work are summarized as follows:
\begin{itemize}
    \item We prove that it is 
    $\mbox{NP}$-Complete to decide if there exists a \ac{pdfa} that is consistent with a given sample and its size, measured by the number of states, is at most $k$ ( Section~\ref{sec:hardness}). This shows that the preference learning problem is computationally challenging. 
    \item We present an algorithm for learning a \ac{pdfa} from a given  sample (Section~\ref{sec:algothim}). The algorithm begins by learning a partial order from the sample and then applies a variant of state merging algorithm on a prefix tree automaton made from the words within the sample.
    \item We define the properties of a  characteristic sample for learning a \ac{pdfa} and prove that, if the sample is characteristic,   our algorithm learns the minimal \ac{pdfa} equivalent to the original \ac{pdfa} from which the sample was generated (Section~\ref{sec:iden_in_limit}). 
\end{itemize}
In Section~\ref{sec:case_studies}, we experimentally demonstrate the correctness and efficacy of our algorithm.
Finally, in Section~\ref{sec:conc}, we conclude our paper and discuss future directions.
 
\section{Preliminaries}
\label{sec:prel}
\textbf{Notations} Let $\Sigma$ be a finite set of symbols called \emph{alphabet}. A finite word is a finite sequence of symbols in $\Sigma$.
The set of all finite words over $\Sigma$ is denoted $\Sigma^\ast$. The set of all finite words over $\Sigma$ with length $k$ is denoted $\Sigma^k$. The empty string,  the string with length $0$, is   denoted $\epsilon$.  
    A language over $\Sigma$ is a set $L \subseteq \Sigma^\ast$. 
     Given $L \subseteq \Sigma^\ast$, $\Pref(L)= \{u \in \Sigma^\ast \mid \exists v \in \Sigma^\ast, uv\in L\}$ is a set of prefixes of words in the set $L$. Given  $\Sigma^\ast$, $<_L$  is the lexicographical ordering over $\Sigma^\ast$.

    In this work, we are interested in preferences over temporally extended goals. Each temporally extended goal is equivalent to a regular language, which is a subset of $\Sigma^\ast$ that can be represented by   a   \ac{nfa}, defined next. The temporally extended goal can also be translated from a subclass of linear temporal logic formulas \cite{pnueli1977temporal}. We refer the reader to \cite{kupferman2001model} for more details in the specification formal languages and to \cite{rahmani2023probabilistic,kulkarni2022opportunistic} the specification of preference over (a subclass of) temporal logic formulas.

    %

    \begin{definition}
    \label{def:nfa}
        An \ac{nfa} is a tuple $\nfa = (Q, \Sigma,\delta, q_0, F)$ in which $Q$ is a finite set of states, $\Sigma$ is the alphabet, $\delta: Q \times \Sigma \rightarrow 2^Q$ is the  transition function, which outputs a set of possible next states given an input symbol from a given state. $q_0$ is the initial state, and $F \subseteq Q$ is the set of accepting (final) states.
    \end{definition}
 
     For an NFA,  the transition function $\delta: Q \times \Sigma \rightarrow 2^Q$ can be viewed as a relation $\delta \subseteq Q \times \Sigma \times Q$ where $(q, a, q') \in \delta$ iff $q' \in \delta(q, a)$.
     The extended transition function is defined as usual: $\delta: Q \times \Sigma^\ast  \rightarrow 2^Q$, that is, for each $q \in Q$, $a \in \Sigma$, and $w \in \Sigma^*$, $\delta(q, aw) = \bigcup_{q' \in \delta(q, a)} \delta(q', w)$, and that $\delta(q, \epsilon) = \{q\}$.
    A word $w \in \Sigma^*$ is accepted by the \ac{nfa} iff $\delta(q_0, w) \cap F \neq \emptyset$.
    The language of $\nfa$, denoted $L(\nfa)$, is the set of all words accepted by $\nfa$, i.e., $L(\nfa) = \{ w \in \Sigma^* \mid \delta(q_0, w) \cap F \neq \emptyset \}$.

        A \ac{dfa} is a special case of an NFA  $\nfa = (Q, \Sigma,\delta, q_0, F)$ in which for each $q \in Q$ and $a \in \Sigma$, $|\delta(q, a)| = 1$. For simplicity, the transition system of the \ac{dfa} written as a complete function $\delta: Q \times \Sigma \rightarrow Q$ and extended as $\delta(q, aw)=\delta(\delta(q,a), w)$ for $q \in Q$, $a \in \Sigma$, and $w\in \Sigma^\ast$. 


    %
       
  Next, we present a formal model for specifying preferences and then draw the connection between this model to a \ac{dma}.

    \begin{definition}
		\label{def:model_preference}
  Given a countable set $U$ of outcomes, a preference model over $U$ is a tuple $\langle U, \succeq \rangle$ where $\succeq$ is a \emph{preorder}, \ie, a reflexive and transitive   relation, on $U$.
	\end{definition}
     Given $u_1, u_2 \in U$, we write $u_1 \succeq u_2$ if $u_1$ is \emph{weakly preferred to} (\ie,  is at least as good as) $u_2$; and $u_1\sim u_2$ if $u_1\succeq u_2$ and $u_2\succeq u_1$,
    that is,  $u_1$ and $u_2$ are \emph{indifferent}.
    We write
    $u_1 \succ u_2$ to mean $u_1$ is \emph{strictly preferred} to $u_2$, \ie, $u_1\succeq u_2$ and $u_1\not \sim u_2$, and use
	$u_1 \nparallel u_2$ to indicate $u_1$ and $u_2$ are \emph{incomparable}, i.e., $u_1 \not\succeq u_2$ and $u_2 \not\succeq u_1$.

    %
   A preference over finite words, with $U\coloneq \Sigma^\ast$, can be specified using a discrete structure defined as follows. 
    %

     \begin{definition}[\ac{pdfa} extended from \cite{rahmani2023probabilistic}]
		\label{def:pdfa}
		A \ac{pdfa} over an alphabet $\Sigma$ is a tuple $\pdfa= \langle Q, \Sigma, \delta, q_0, O, \succeq, \lambda \rangle$ in which 
        $Q$ is a finite set of states, $\Sigma$ is the alphabet, $\delta: Q \times \Sigma \rightarrow Q$ is the deterministic transition function, $q_0$ is the initial state,
        $O$ is a finite set of \emph{ranks}, and $\succeq\subseteq O^2$ is a \emph{partial order}, which is a reflexive, transitive, and anti-symmetric relation, on $O$, and $\lambda: Q \rightarrow O$ is the \emph{ranking} function, which maps each state $q\in Q$ to a unique rank $o \in O$ and is a complete and surjective function.
	\end{definition} 
        Note that this definition replaces the set of accepting (final) states of a DFA with the partial order relation on accepting conditions.
        A \ac{pdfa} is readily visualized as a finite automaton along with a directed acyclic graph showing the partial order.

		The \ac{pdfa} encodes a preference model $\succeq$ for $\Sigma^\ast$ as follows. 
		Consider two words $w, w' \in \Sigma^\ast$, and let $q = \delta(q_0, w)$ and $q' = \delta(q_0, w')$.
        There are four cases:
        (1) if $\lambda(q) = \lambda(q')$, then $w \sim w'$; (2) if $\lambda(q)  \succ  \lambda(q')$, then $w \succ w'$; (3) if $\lambda(q') \succ \lambda(q)$, then $w' \succ w$; and (4) otherwise, $w \nparallel w'$.
        A preference model $\langle \Sigma^\ast, \succeq \rangle$ is called \emph{regular} if and only if it can be represented by a \ac{pdfa}.

        Two \ac{pdfa}s $\pdfa= \langle Q, \Sigma, \delta, q_0, O, \succeq, \lambda \rangle$ and 
        $\pdfa'= \langle Q', \Sigma, \delta', q_0', O', \succeq', \lambda' \rangle$ are \emph{equivalent} if they encode the same preorder over $\Sigma^*$.
        Formally, they are equivalent if for any $w_1, w_2 \in \Sigma^*$, assuming $q_1 = \delta(q_0, w_1)$, $q_2 = \delta(q_0, w_2)$, $q_1' = \delta'(q_0', w_1')$, and $q_2' = \delta'(q_0', w_2')$, (1) $\lambda(q_1) = \lambda(q_2)$ iff $\lambda'(q_1') = '\lambda'(q_2')$, (2) $\lambda(q_1) \succ \lambda(q_2)$ iff $\lambda'(q_1') \succ' \lambda'(q_2')$, and (3) $\lambda(q_2) \succ \lambda(q_1)$ iff $\lambda'(q_2') \succ' \lambda'(q_1')$.
        
        %

A \ac{pdfa} $\pdfa$ is \emph{canonical} if for any $\pdfa'$ equivalent to $\pdfa$, the number of states in $\pdfa$ is no greater than the number of states in $\pdfa'$.
        
        A \ac{pnfa} $\pdfa= \langle Q, \Sigma, \delta, q_0, O, \succeq, \lambda \rangle$ generalizes a \ac{pdfa} by allowing its transition function to be nondeterministic, i.e., $\delta: Q \times \Sigma \rightarrow 2^Q$, and its ranking function to be incomplete.
        When the ranking of a state $q$  is undefined, it is denoted $\lambda(q) =  \undefined$.

        \begin{example}
\label{ex:pdfa1}
An example of \ac{pdfa} is shown in  Fig.~\ref{fig:pdfa}. 
The initial state is $00$, indicated with the incoming arrow.
The colors $\{ \text{\textbf{b}(lue)}, \text{\textbf{o}(range)}, \text{\textbf{g}(reen)}\}$ are ranks and the ranking function $\lambda(\cdot)$ is indicated by the coloring of states. For example, $\lambda(00) = \mathbf{b} $. 
        The directed graph on the right side represents a partial order over the ranks $\{\mathbf{b}\succ \mathbf{o}, \mathbf{g}\succ\mathbf{o}\}$. By this ordering, we have $bb \succ a$ because $\lambda(\delta(00,bb))= \lambda(00)=\mathbf{b}$ and 
$\lambda(\delta(00,a))= \lambda(10)=\mathbf{o}$ and $\mathbf{b}\succ\mathbf{o}$.
        \begin{figure}[h]
		\centering
		\includegraphics[width=0.8\linewidth]{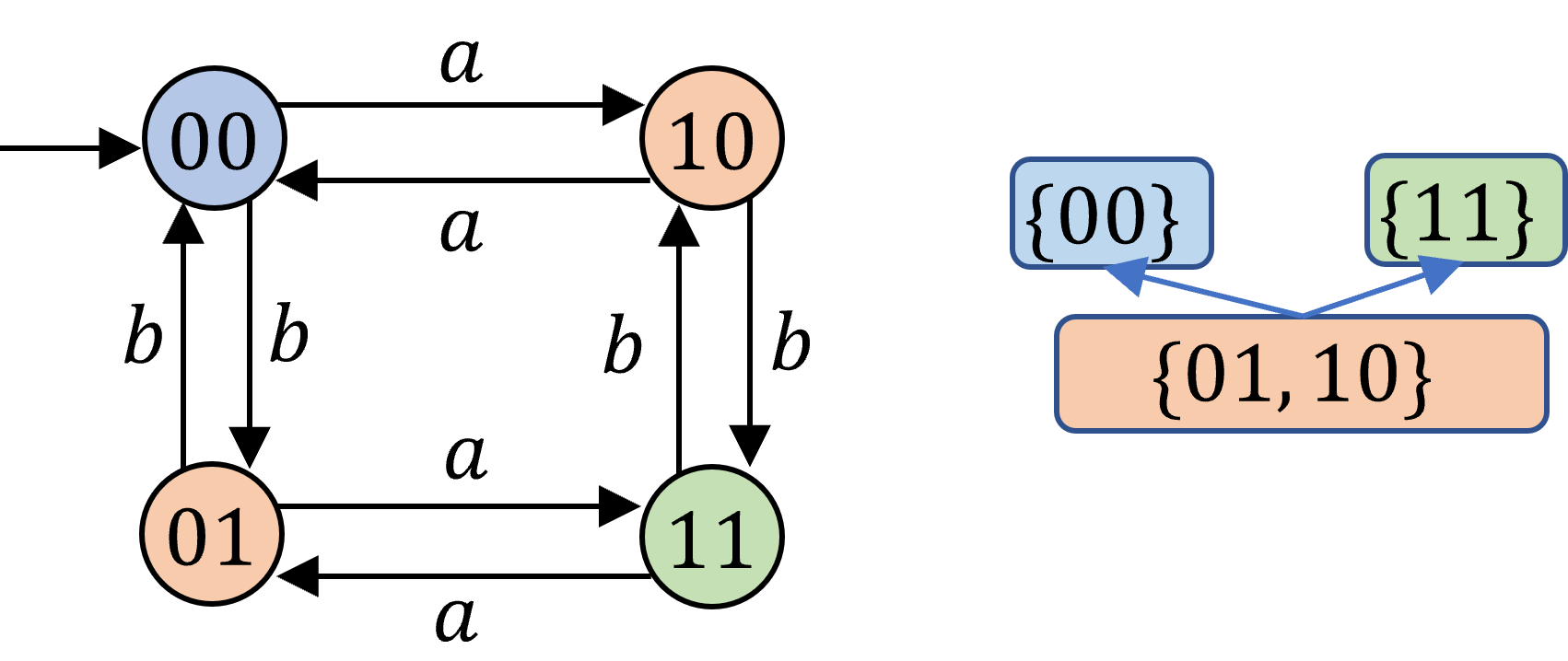}
		\caption{
		      A \ac{pdfa} with states $\{00,10,01,11\}$ and alphabet $\Sigma = \{a,b\}$  with the ranking function presented by a directed graph (right). 
		}
		\label{fig:pdfa}
	\end{figure}
 \end{example}

          \begin{definition}
        \label{def:sample}
            A finite set $S \subseteq \Sigma^\ast \times \Sigma^\ast \times \{0, 1, \perp \}$  is called a \emph{preference sample}, or simply a \emph{sample}, where for each 
        $(w, w', b) \in S$, 
        \begin{enumerate}
            \item $b  = 0$ means $w$ and $w'$ are indifferent to each other,
            \item $b   = 1$ means $w$ is strictly preferred to $w'$, and
            \item $b  = \perp$ means $w$ and $w'$ are incomparable.
        \end{enumerate}
        \end{definition}
        %
        A sample is readily visualized as a `multi-edge-sets' graph with three sets of edges, one set of solid undirected edges for representing
        the pairs of indifferent words, one set of solid directed edges for representing the strict preference relation, and 
        one set of undirected and dashed edges for representing the incomparable pairs of words.  The set of words that appear in $S$ is denoted $W_S$, formally defined as $W_S = \{w \in \Sigma^* \mid \exists w' \in \Sigma^\ast, b\in \{0,1,\perp\}: (w, w', b) \in S \text{ or } (w', w, b) \in S \}$.

        \begin{definition}
        \label{def:consistent}
           A \ac{pdfa} $\pdfa= \langle Q, \Sigma, \delta, q_0, O, \succeq, \lambda \rangle$ and a sample $S$ are \emph{consistent} with each other, if for each
           $(w, w', b) \in S$, 
            \begin{itemize}
                \item if $b  =0$, then $w \sim w'$;
                \item if $b = 1$, then $w \succ w'$; and
                \item if $b = \perp$, then  $w \nparallel w'$.
            \end{itemize}
        \end{definition}

        \begin{example}[cont.]
            
     Figure~\ref{fig:sample} shows  an example of a sample $S$ consistent with the \ac{pdfa} in Example~\ref{ex:pdfa1}. For example, the directed edge $a \rightarrow ba$ means $ba \succ a$. This is consistent with the \ac{pdfa} in Fig.~\ref{fig:pdfa} because $\lambda (\delta(00, a)) =\lambda(10)=\textbf{o}$ and $\lambda (\delta(00, ba)) =\lambda(11)=\textbf{g}$ and $\textbf{g}\succ \textbf{o}$.

    \begin{figure}[h]
		\centering
		\includegraphics[width=\linewidth]{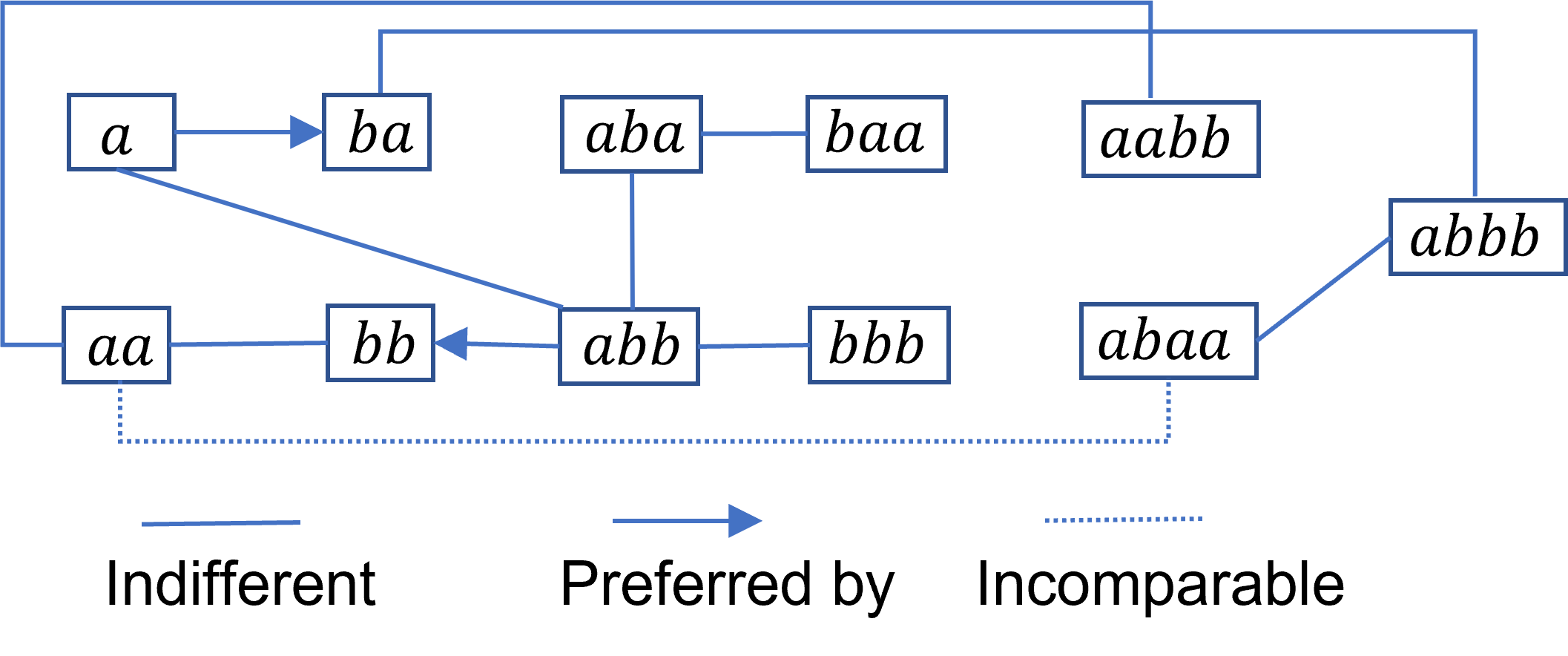}
		\caption{
		      A preference sample for alphabet $\Sigma = \{a, b\}$. 
       The lacking of an edge between two words  means the sample does not include any comparisons between those two words.
        }
		\label{fig:sample}
	\end{figure}
         \end{example}

        Having these definitions, we present our problem formulation.
        \problembox{Learning Preference Model (\LPM)}
{A sample $S=\{ (w_1, w_1', b_1), \cdots, (w_n, w_n', b_n) \}$, which is drawn from an uknown \ac{pdfa} $\pdfa$.}
{
    A \ac{pdfa} equivalent to $\pdfa$.
}

\section{Hardness of learning preference model}
\label{sec:hardness}
A regular preference model may be represented by more than one  \ac{pdfa}s.
 Among them, there is a   \ac{pdfa}  that has the smallest number of states, referred to as the \emph{canonical \ac{pdfa}}.
Ideally, for a given sample $S$, we would like to find a canonical \ac{pdfa}  consistent with $S$. 

        To understand the complexity of learning a \ac{pdfa} from a given sample, we first consider the following problem.
        %
        %
        %
        \decproblembox{Minimum Consistent PDFA (\MCPDFA)}
        {A sample $S=\{ (w_1, w_1', b_1), \cdots, (w_n, w_n', b_n) \}$ and a positive integer $k$.}
        {\Yes  if there exists a \ac{pdfa} with at most $k$ states that is consistent with $S$, otherwise \No.}
        Note if the canonical \ac{pdfa} from which the sample is drawn has at most $k$ states, then the answer to this decision problem is \Yes.
        However, if the answer to this is problem is \Yes, it is still unknown whether the canonical \ac{pdfa} has at most $k$ states or not.
        Yet, understanding the hardness of this problem sheds lights on how difficult is to learn an exact regular preference model from which a sample is drawn.
        
         Regarding this decision problem, we present the following result.
         \begin{lemma}
		\label{lem:NP}
        \MCPDFA $\in \NP$.
		\end{lemma}
		\begin{proof}
            We need to show that given a \ac{pdfa} $\pdfa= \langle Q, \Sigma, \delta, q_0, O, \succeq, \lambda \rangle$ as a certificate, we can verify in polynomial time to size of $S$ both (1) whether $\pdfa$ has at most $k$ states, and that (2) $\pdfa$ is consistent with $S$.
		    Trivially, in constant time we can check (1).
            Assume the transition and the ranking functions are encoded as hash tables.
            If not, one can clearly construct those hash tables in polynomial time.
            In respect with (2), for each $(w_i, w_i', b_i) \in S$, in time $\bigO(|w_i|+|w_i'|)$ we can obtain $\delta(q_0, w_i)$ and $\delta(q_0, w_i')$ and then in constant time we can check whether $(w_i, w_i', b_i)$ is consistent with $\pdfa$, using the rules in Definition~\ref{def:consistent}. 
            Accordingly, checking whether $S$ is consistent with the \ac{pdfa} takes $\bigO(\sum_{1 \leq i \leq n}(|w_i|+|w_i'|))$, which is a polynomial time to the size of $S$.
        \end{proof}

        To further study the hardness of \MCPDFA, we first consider the classical problem of learning a minimal consistent deterministic finite automaton from positive and negative data.
        \decproblembox{Minimum Consistent DFA \MCDFA}
        {A finite non-empty set $S^+ \in 2^{\Sigma^\ast}$, a finite non-empty set $S^- \in 2^{\Sigma^\ast}$, and an integer $k'$}
        {\Yes if there exists a \ac{dfa} $A=(Q, \Sigma, \delta, q_0, F)$ such that $S^+ \subseteq L(A)$, $S^- \cap L(A) = \emptyset$, and $ |Q| \leq k'$, and \No otherwise.  } 
        %
        \begin{proposition}
    		\label{pro:NP-hard}
            \MCDFA $\in \NP$-hard.
        \end{proposition}
        \begin{proof}
            See the proofs in \cite{gold1978complexity,lingg2024learning}.
        \end{proof}
        We use this result to prove that the preference learning problem \MCPDFA is computationally hard.
        \begin{theorem}
        \label{thr:NP-hard}
            \MCPDFA $\in \NP$-hard.
        \end{theorem}
        \begin{proof}
            By reduction from \MCDFA.
            Given an instance $$x = \langle S^+, S^-, k \rangle$$ we construct an instance $$y = \langle S, k \rangle,$$ where
            \begin{multline}
                S = \{ (w, w', 0) \mid w, w' \in S^+ \} \cup \{ (w, w', 0) \mid w, w' \in S^- \} \\ 
                \cup \{ (w, w', 1) \mid w \in S^+ \text{ and } w' \in S^- \} \notag.
            \end{multline}
            We need to show that (1) the reduction takes a polynomial time, and (2) the reduction is correct.
            It holds that $\abs{S} = \abs{S^+}^2+\abs{S^-}^2+\abs{S^+}\abs{S^-}$, and clearly, the reduction takes a polynomial time.

            For the correctness of the reduction, we show that \MCDFA produces \Yes for instance $x$ if and only if \MCPDFA produces $\Yes$ for instance
            $y$.

            ($\Rightarrow$) Suppose \MCDFA produces \Yes for instance $x$, i.e., there exists a DFA $A = (Q, \Sigma, \delta, q_0, F)$ such that $|Q| \leq k$, $S^+ \subseteq L(A)$, and $S^- \cap L(A) = \emptyset$.
            We construct \ac{pdfa} $\pdfa = (Q, \Sigma, \delta, q_0, O, \succeq, \lambda)$ in which $O= \{1, 2\}$, $\succeq = \{(1, 1), (2, 2), (1, 2)\}$, and
            for each $q \in Q$, let $\lambda(q) = 1$ if $q \in F$, and $\lambda(q) = 2$ otherwise.
            By the construction of $S$, for each $(w, w', b) \in S$, it holds that $b \in \{0, 1 \}$ and
            \begin{itemize}
                \item If $b=0$, then it means  $\lambda(\delta(q_0, w)) = \lambda(\delta(q_0, w'))$, and thus
                 $w \sim w'$;
                 \item If $b=1$, then it means $\lambda(\delta(q_0, w)) = 1$ and $\lambda(\delta(q_0, w'))=2$, and because $1 \succeq 2$, $w \succ w'$.
            \end{itemize}
            Therefore, $\pdfa$ is consistent with $S$, and because $\pdfa$ has $k$ states, \MCPDFA produces \Yes for $y$.

            ($\Leftarrow$) Suppose \MCPDFA produces \Yes for instance $y$, i.e., there exists a \ac{pdfa} $\pdfa = (Q, \Sigma, \delta, q_0, O, \succeq, \lambda)$ consistent with $S$ such that $|Q| \leq k$.
            Because $S^+$ and $S^-$ are non-empty, there exists a tuple $(w, w', 1) \in S$. 
            This means that $O$ contains at least two distinct ranks $o_1$ and $o_2$ such that $\lambda(\delta(q_0, w))=o_1$, $\lambda(\delta(q_0, w'))=o_2$, and  either  $o_1 \succ o_2$ or $o_2 \succ o_1$.
            %
            Additionally, because $S$ is consistent with $\pdfa$, it must be the case that for all $w \in S^+$, $\lambda(\delta(q_0, w))= o_1$,   for all
            $w' \in S^-$, $\lambda(\delta(q_0, w'))=o_2$, and that $o_1 \succ o_2$. 
            Let $F = \{q \in Q \mid \lambda(q) = o_1 \}$.
            Accordingly, for the DFA $A = (Q, \Sigma, \delta, q_0, F)$, $S^+ \subseteq L(A)$ and that $S^-  \cap  L(A) = \emptyset$.
            Because DFA $A$ is consistent with $S^+$ and $S^-$ and that $|Q| \leq k$, the answer of \MCDFA given $x$ is \Yes if the answer of \MCPDFA to $y$ is
            \Yes.
            This completes the proof.
        \end{proof}

        Accordingly, we the following result holds.
        \begin{theorem}
        \label{thr:NP-Complete}
            \MCPDFA $\in \NP$-Complete.
        \end{theorem}
        \begin{proof}
            It follows from Lemma~\ref{lem:NP} and Theorem~\ref{thr:NP-hard}.
        \end{proof}
        
        The impact of this result is that assuming $\P \neq \NP$, there exists no algorithm that  can efficiently compute the minimum \ac{pdfa} consistent with an arbitrary sample $S$. 

        Theorem~\ref{thr:NP-Complete} shows that our automata inference problem, the \MCPDFA problem, is generally hard. 
        In the next section, we present an efficient algorithm that  learns the canonical presentation of the true \ac{pdfa} from which the sample is generated, provided the sample satisfies certain conditions. 

\section{Algorithm Description}
\label{sec:algothim}
        %
        In this section, we present an algorithm to learn a \ac{pdfa} from a given sample.
        This algorithm first partitions the words within the sample 
        and derives a partial order   between the subsets, referred to as blocks, within the partition.
        %
        %
        The algorithm then constructs a prefix tree automaton (defined next) from those words appeared in the sample, and extends the RPNI algorithm~\cite{oncina1992identifying}, 
        to iteratively merge the states of the prefix tree automaton until all the states of the tree automaton are processed.
        %
        

We introduce some notations: Let $W \subseteq \Sigma^\ast $ be a set and $\pi =\{B_1,\ldots, B_{\abs{\pi}}\}$ be a partition of the set $W$, that is, $B_i \neq \emptyset$ for all $i$'s, $\cup_{i=1}^{\abs{\pi}} B_i = W$, and $B_i\cap B_j=\emptyset$ if $i\ne j$.

\begin{definition}
  Given a partition $\pi$ over a set $W$,  for each $w\in W$, the \emph{block} in partition $\pi$ that contains $w$ is denoted by $B(w, \pi)$.
\end{definition}

%
        
        \begin{definition}
        \label{def:eq_graph}
  Given a sample $S \subset \Sigma^* \times \Sigma^* \times \{0, 1, \perp\}$,   the \emph{indifference graph} for   sample  $S$ is an undirected graph $G_{\sim}=(W_S, E_{\sim})$ where for each $w, u \in W_S$, 
        $\{w, u\} \in E_\sim$ iff $\{(w, u, 0), (u, w, 0)\} \cap S \neq \emptyset$.
        \end{definition}
      Note that even though the sample might not directly compare two words $u,w$, we may derive these two words to be indifferent  using the transitive closure of this relation.
        %
        %
        %
        %
        We let $\pi_S$ to be the set of strongly connected components (SCCs) of $G_\sim$, i.e., $\pi_S=\SCCs(G_\sim)$. 
        Clearly, $\pi_S$ is partition of $W_S$. Also, for each $w\in W_S$,  any $ w' \in B(w, \pi_S )$ is indifferent to $w$ by construction. 
        Using these, we create a partial order in the following.

 Let $\rankS$ be the index set of $\pi_S$, \ie,  $\rankS =\{1,\ldots,|\pi_S| \}$. Let $C_i$ be the $i$-th SCC in $\pi_S$.
        We construct a relation $R_S$ over the set of indices $\rankS$ as follows:
        \begin{equation}
        \label{eq:relation-rank}
            R_S = \{(i, j) \mid \exists w \in C_i, w' \in C_j : (w, w', 1) \in S \}
        \end{equation}
          
The partial order, denoted $\succeq_S$, is constructed as the union of the reflexive closure and the transitive closure of $R_S$.

    \begin{definition}
    \label{def:pt} Given a sample $S$, its indifference graph $G_{\sim}$, the index set $\rankS$ of SCCs and relation $R_S$ constructed from $G_\sim$,  a prefix tree \ac{pdfa} $\PTree_S = (\Pref(W_S)=\{u_0,\ldots, u_r\}, \Sigma, \delta_S, u_0 \coloneqq \epsilon, O\coloneqq \rankS, \succeq_S, \lambda_S)$ 
    in which 
    \begin{itemize}
    \item $\Pref(W_S)$ is a set of states in the prefix tree \ac{pnfa}, which is the set of all prefixes of the words appeared in the sample. In particular, let  $u_0, u_1, \cdots, u_r$ where $r = |\Pref(W_S)|$ be the  states within $  \PTree_S$ indexed with the lexicographical ordering.
    \item $\Sigma$ is the alphabet.
    \item The transition function $\delta_S$ is defined as follows:
      For each $w \in \Pref(W_S)$, $a \in \Sigma$ and $wa \in \Pref(W_S)$, $\delta_S(w, a)=wa$. 
    \item $u_0 = \epsilon$, the empty string, is the initial state.
    \item $\rankS$ is the set of ranks.   
    \item $\succeq_S$ is the partial order relation on $\rankS$.
    \item $\lambda_S: \Pref(W_S) \rightarrow \rankS \cup \{\varnothing\}$ is defined such that 
each prefix $u \in \Pref(W_S)  \cap W_S $ is mapped to the index $i$ of the SCC that contains $u$. That is, if $C_i\in \pi_S$ contains $u$, then $\lambda_S(u)=i$.  For each $u \in \Pref(W_S) \setminus W_S$, $\lambda_S(u)=\varnothing$, meaning   $\lambda_S$ is undefined for $u$. 
     \end{itemize}
     \end{definition}
     By construction of SCCs, the rank of each state in the prefix-tree automaton is unique.
     
Slightly abusing notation, we use each subset (block) of the partition $\rankS$ to denote a rank. It can be understood as each subset is associated with a unique index and the ranks are the set of indices.

        \begin{example}[cont.]%
        The top graph in 
        Figure~\ref{fig:learn_pref_graph} illustrates the indifference graph  $G_\sim$ constructed for the sample in Figure~\ref{fig:sample}. We let words with the same colors be in the same strongly connected component. The bottom directed graph shows the partition $\pi_S$ constructed using the SCCs and the relation $R_S$ is represented by directed edges. For clarity, we omitted the index set and just directly use the SCCs.  For example,  $ \{a,aba, abb, baa, bbb\} \rightarrow \{aa,bb, aabb\} $, because $bb \in   \{aa,bb, aabb\}$, $abb \in \{a,aba, abb, baa, bbb\}$, and  $(bb, abb, 1)\in S$.
      
        %
        %
         %
        \begin{figure}[h]
		\centering
\includegraphics[width=\linewidth]{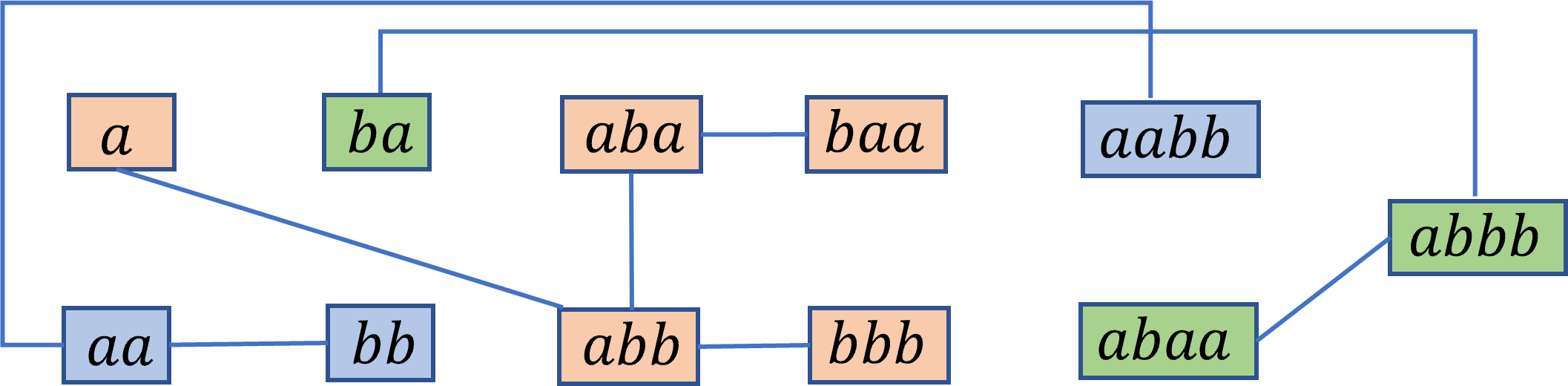}
          \vspace{2ex}
          
          \includegraphics[width=0.6\linewidth]{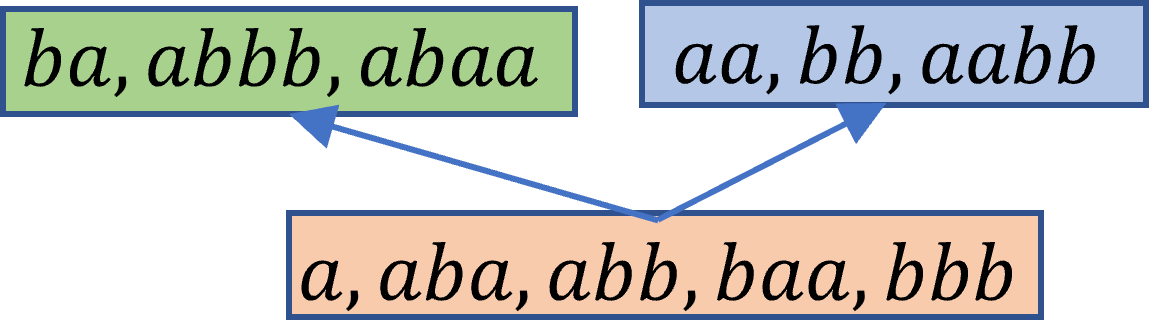}
		\caption{
	\textbf{Top)} The indifference graph for the sample in Figure~\ref{fig:sample}. \textbf{Bottom)} The partition $\rankS$ and the relation $R_s$.
.}
\label{fig:learn_pref_graph}
	\end{figure}
        \end{example}

    Figure~\ref{fig:A0} shows the prefix tree \ac{pnfa} constructed for the sample in Figure~\ref{fig:sample}.
    The figure uses colors to differentiate the ranks of the states, \ie, the indices of SCCs.
    The uncolored states are those who do not appear in the sample. The ranking function is undefined for those states. 

   To learn a \ac{pdfa} from sample $S$, we develop a state-merging operation similar to  the RPNI algorithm~\cite{oncina1992polynomial}, which is introduced for learning \ac{dfa}. 
The algorithm employs a \emph{quotient} operation on automata. 
\begin{definition}[Quotient of a \ac{pnfa}]
Given a \ac{pnfa} $\pnfa = (Q, \Sigma, \delta, q_0, O, \succeq, \lambda)$, a partition $\pi$ of state set $Q$   is \emph{consistent} with the \emph{ranking} $\lambda$ of $\pnfa$ if for any $B\in \pi$, for any $q, q'\in B$,  either 1) $\lambda(q) = \lambda(q')$  or 2) $\lambda(q) = \undefined$ or $\lambda(q') = \undefined$. Such a partition $\pi$ is said to be  \emph{ranking-consistent}.
The quotient of $\pnfa$ under a ranking-consistent partition $\pi$, denoted $\pnfa / \pi $, is a \ac{pnfa} $$\pnfa/\pi = (\pi, \Sigma, \delta', B(q_0, \pi), O,  \succeq', \lambda')$$ where \begin{itemize}
    \item $\pi$ is the set of states. Each state $B\in \pi$ is a subset of $Q$. 
    \item $\Sigma$ is the alphabet.
    \item the transition relation $\delta'$ is defined by:
   Given $B, B' \in \pi$, and $a \in \Sigma$, $(B, a, B') \in \delta'$ iff there are states $q \in B$ and $q' \in B'$ such that $\delta(q, a) = q'$. 
   That is, if there is a state $q$ in the block $B$, and the transition given input $a$ from $q$ is defined,  then a  transition from $B$  to $B'$ with input $a$ is defined, where $B'$ includes all states in $Q$ that can be reached with input $a$ from some state in $B$. 
    \item $O$ is the set of ranks. 
    \item $\succeq' \coloneqq \succeq$ is the partial order over ranks and is the same as in $\pnfa$.
 \item   $\lambda'$ is the ranking function, defined by: For each $B\in \pi$,  if there exists $q\in B$ such that $\lambda(q)\ne \undefined$, then $\lambda'(B) = \lambda(q)$,  otherwise, $\lambda'(B) = \undefined$.
\end{itemize}
\end{definition}
It is noted that for a given $B\in \pi$ and an input $a\in \Sigma$, there may exist more than one subset $B'$ to be reached, \ie, it is possible that $(B,a,B')\in \delta'$ and $(B,a,B'')\in \delta'$ where $B'\ne B''$. Thus, even if $\calA$ is deterministic, the quotient of $\calA$ under $\pi$ may be nondeterministic. 

It is also noted that due to $\pi$ being ranking-consistent, the ranking function $\lambda'$ of the quotient of $\calA$ is well-defined.

\begin{definition}
Given a partition $\pi$ of $Q$, and two blocks $B_i, B_j \in \pi$,  the \emph{join} of $B_i$ and $B_j$, denoted $J(\pi, B_i, B_j)$, is the partition obtained from $\pi$ by combining $B_i$ and $B_j$, i.e.,
        \begin{multline*}
            J(\pi, B_i, B_j) =  \left( \pi\setminus \{B_i, B_j\} \right) \cup 
                \{B_i \cup B_j\}. \notag
        \end{multline*}
\end{definition}
        This `join' operation is also called \emph{merging} in the context of automata learning.

Given a quotient  $\pnfa/\pi$, the \emph{determinization} of  $\pnfa/\pi$, denoted as $D(\pnfa/\pi)$, is computed \emph{recursively} given three different cases:
Case 1: Suppose there exists $B_k \in \pi$ and $a\in \Sigma$ such that   $B_i, B_j \in \delta(B_k,a)$ but $B_i\ne B_j$,  and either  1) $\lambda(B_i)= \undefined$ or 2) $\lambda(B_j)= \undefined$ or 3) $\lambda(B_i)=\lambda(B_j)$, then
\begin{align*}
& D(\pnfa /\pi) \coloneqq \\
&\begin{cases} D(\pnfa /J(\pi,B_i, B_j)) & \text{ if  $D(\pnfa /J(\pi,B_i, B_j)) \ne\varnothing$,}\\
\varnothing &\text{otherwise},
\end{cases}
\end{align*}
where $J(\pi,B_i, B_j)$ is the join of $B_i$ and $B_j$ in partition $\pi$ and $\varnothing$ means that the output is undefined.

Case 2:   Suppose there exist $B_k \in \pi$ and $a\in \Sigma$ such that   $B_i, B_j \in \delta(B_k,a)$ but $B_i\ne B_j$,   $\lambda(B_i)\ne  \undefined$, $\lambda(B_j)\ne \undefined$, and $\lambda(B_i)\ne \lambda (B_j)$, then 
\[
D(\pnfa /\pi) =  \undefined.
\]
Case 3: if there is no non-deterministic transitions in $\pnfa /\pi$, that is, $|\delta(B, a)| \le 1$ for any $B\in \pi$, $a\in \Sigma$, then 
\[
D(\pnfa /\pi) =  \pnfa /\pi.
\]

It is observed that the function $D(\cdot)$ operates on a \ac{pnfa} and is a recursive function (self-referential). It is similar to determinization of an \ac{nfa}, except that each state here is a block and the join of two blocks is carried out when the ranking consistency condition is satisfied. That is, If $B_i$ is joined with $B_j$ in a ranking-consistent partition $\pi$ in this determinization operation, then the new partition  $J(\pi, B_i, B_j)$ is ranking-consistent.

The joining and determinizing operations repeat until no more joins can be made or the input to function is a \ac{pdfa}.

\normalcolor


 Accordingly, we define a new operation $DJ$, called the \emph{deterministic join}, which produces a quotient automaton:
 $$
 \pnfa_0 / DJ(\pi, B_i, B_j) = D(\pnfa_0 / J(\pi, B_i, B_j)).
 $$
 In words, this operation first joins two blocks $B_i$ and $B_j$, and then may  join more blocks recursively as the result of the determinization of the quotient of $\pdfa_0$ under the new partition $J(\pi, B_i, B_j)$. After the operation, the obtained partition is denoted by $DJ(\pi_i,B_i, B_j)$.


Finally,  the \ac{pdfa} learning algorithm is described in Alg.~\ref{alg:euclid}. It uses a function $\eta: 2^{W_S} \setminus \emptyset \rightarrow \Sigma^*$, referred to as  the \emph{block-name} function, such that for each $U \in 2^{W_S} \setminus \emptyset$,
$\eta(U)$ is the smallest word within $U$ given the lexicographical ordering, i.e., $\eta(U)=\min_{<_L}\{w \in U\}$.
%
%
%
We lift the lexicographical ordering $<_L$ on to the blocks of a given partition such that given $B_1, B_2 \in \pi$,
$B_1 <_L B_2$ iff $\eta(B_1) <_L \eta(B_2)$.

Algorithm~\ref{alg:euclid} begins with an initial partition $\pi_0$ (line~\ref{line:initial_parition}), and then, in lexicographical ordering, iteratively performs joining of blocks and determinization of the automaton to update the partition for $n$ steps where $n$ is the number of states in the $\PTree_S$. At the end, it returns the partition $\pi_r\coloneqq \pi_n$ and the quotient \ac{pdfa} $\calA_r \coloneqq D(\calA_0/\pi_r)$.

\begin{algorithm}
\caption{LearnPDFA}\label{alg:euclid}
\begin{algorithmic}[1]
\Procedure{LearnPDFA}{$\PTree_S$} 
\State $\pi_0 \gets \{ \{u_0\}, \{u_1\}, \cdots, \{u_n\} \}$ 
\label{line:initial_parition}
\State $\calA_0 \gets \PTree_S$.
\For{$i \gets 1$ \textbf{to} $n$} \label{line:outer_loop_start}
\State $\Merged \gets \False$
\If{$\text{exists } B_i \in \pi_{i-1} \text{ s.t. } u_i = \eta(B_i)$}
\For{$B <_L B_i$ in the ascending order} \label{line:for_loop_start} 
\If{$\lambda_S(B)=\lambda_S(B_i) \text{ or } \lambda_S(B) = \undefined \text{ or } \lambda_S(B_i) = \undefined$}
\If{$ D(\pnfa_0 / J(\pi_{i-1}, B, B_i)) \neq \undefined$} \label{line:check_determinization}
\State $\pi_i = DJ(\pi_{i-1}, B, B_i)$ 
\label{line:join_deter}
\State $\calA_i \gets D(A_0/\pi_i)$,
 \State  $\Merged \gets \True$
\State $\textbf{break}$
\EndIf
\EndIf
\EndFor\label{line:for_loop_end}
\EndIf
\If{$\Merged == \False$}
\State $\pi_i \gets \pi_{i-1}$
\State $\calA_i \gets \calA_{i-1}$

\EndIf
\EndFor \label{line:outer_loop_end}
\State \textbf{return} $\pi_r\gets \pi_n$, $\calA_r  \gets \calA_n $ 
\EndProcedure
\end{algorithmic}
\end{algorithm}

%

\normalcolor
\begin{example}[cont.]
 We describe the algorithm using the running example. 
The prefix tree \ac{pnfa} for this sample is shown in Figure~\ref{fig:A0}. The set of prefixes in the lexicographical order is 
\begin{multline*}
\epsilon, a, b, aa, ab, ba, bb, aab, aba, abb, \\
baa, bbb, aabb, abaa, abbb.
\end{multline*}

In the first iteration, $i=1$,  $u_1 = a$, $\{B\in \pi_0 |B<_L\{a\}\} = \{\epsilon\}\} $.  We try to merge state $a$ with $\epsilon$, \ie,  $J(\pi_0, \{\epsilon\}, \{a\})$. The quotient automaton $\calA / J(\pi_0, \{\epsilon\}, \{a\})$ is shown in Fig.~\ref{fig:u1_u0}. 
These two states can be merged because only one of them is colored (has rank). 

After performing determinization of this automaton in Fig.~\ref{fig:u1_u0}, $D(\pnfa_0 / J(\pi_0, B, \{u_1\}))$ returns $\undefined$ because both states $ \epsilon$ and $ aa$   can be reached by input $a$ from $\epsilon$ but they are assigned different rankings (orange and blue).
Hence, merging $a$ and $\epsilon$ is rejected, and $\pi_1 = \pi_0$.

For iteration $2$, $u_2=b$ and $\{B\in \pi_1  \mid B<_L \{b\}\} = \{\{\epsilon\}, \{a\}\} $. We first try to merge $  b$ with 
  $  \epsilon $.
Those two can be merged because both are uncolored.
The automaton obtained from  merging is   shown in Figure~\ref{fig:u2_u0}.
It is noted that the automaton is non-deterministic and since both  $ a$ and $ ba$ can be reached by $a$ from state $\epsilon$ but have different colors, this merging is   rejected, 
$D(\pnfa_0 / J(\pi_1, \{\epsilon\}, \{b\})) = \undefined$.

Next, we try to merge $\{b\}$ with $\{a\}$.  The automaton $\pnfa_0 / J(\pi_1, \{b\}, \{a\})$, shown in Figure~\ref{fig:u2_u1}, cannot be determinized   due to non-deterministic transitions with $a$ to both $aa$ and $ba$ which have different colors. 
Thus, $\pi_2 = \pi_1=\pi_0$. That is, till this step, all attempted merges are rejected. We return to  the prefix tree automaton.

Next, $i=3$, $u_3=aa$. $\{B\in \pi_2 \mid B<_L \{u_3\}\} = \{\{\epsilon\}, \{a\}, \{b\}\}$. We try to merge $u_3$ with $\epsilon$, $a$, $b$, in this order.
Merging $aa$ with $\epsilon$ results in the automaton in Figure~\ref{fig:u3_u0}. 
In the determinization, $aab$ is merged with $b$, and $aabb$ is merged with $bb$, shown     in Figure~\ref{fig:u7_u2} and Figure~\ref{fig:u12_u6}, respectively.

Next $i=4$ and $u_4 = ab$. $\{B \in \pi_3\mid B<_L \{ab\} \}= \{\{\epsilon\}, \{a\},\{b\}\}$. We try to merge $ab$ with each block respectively, but have to reject every merge operation because the determinization after merging fails. 
Hence, $\pi_4 = \pi_3$. For example, merging $ab$ with $\epsilon$ fails because $\epsilon \xrightarrow{aa} \epsilon$ and $ab\xrightarrow{aa} abaa $ but $\epsilon$ and $abaa$ have different colors.

For $i=5$, $u_5 = ba$,   $\{B \in \pi_4 \mid B<_L \{ba\} \}= \{\{\epsilon\}, \{a\},\{b\}, \{ab\}\}$. $ba$ cannot merge with either $\epsilon$ or $a$ due to different colors. It cannot be merged with $b$ because $b\xrightarrow{a} ba$ and $ba\xrightarrow{a} baa$ reach two states with different colors. Yet $ba$ can be merged with $ab$.  The resulting automaton is Figure~\ref{fig:u5_u4}. After the determinization, $baa$ is merged with $aba$.

Figures~\ref{fig:u6_u0}-\ref{fig:u9_u1} shows the resulting automaton for the next three iterations.  
Finally, the automaton in Figure~\ref{fig:u9_u1} coupled with the sample partial order in Figure~\ref{fig:learn_pref_graph} is obtained from this algorithm. %
Table~\ref{tbl:final-ex} summarizes the states in $\PTree_S$ being merged:
\begin{table}[h]
\centering
\begin{tabular}{c|c}
      State   shown in Figure~\ref{fig:u9_u1}     &          Merged states         \\
      \hline
$\epsilon$ & $\epsilon, aa,bb,aabb$        \\
$a$        & $a, abb $             \\
$b$        & $b, aba,baa, bbb,aab$   \\
$ab$       & $ab, ba,abaa, abbb$     
\end{tabular}
\caption{The final partition $\pi_r$.}
\label{tbl:final-ex}
\end{table}

\begin{figure*}[h!]
		\centering
		\begin{subfigure}[b]{0.326\textwidth}
			\centering
\input{ex-pdfa-learn/ex-a}			
   \caption{The prefix tree \ac{pnfa} $\pnfa_0$.}
   \label{fig:A0}
		\end{subfigure}
		\hfill
  \centering
		\begin{subfigure}[b]{0.326\textwidth}
			\centering
\input{ex-pdfa-learn/ex-b}			
   \caption{ $\calA_0/J(\pi_0, \{a\},\{\epsilon\})$ [rejected merge].}
     \label{fig:u1_u0}
		\end{subfigure}
		\hfill
\begin{subfigure}[b]{0.326\textwidth}
			\centering
\input{ex-pdfa-learn/ex-c}			
   \caption{ $\calA_0/J(\pi_0, \{b\},\{\epsilon\})$ [rejected merge], $\pi_1 = \pi_0$.}
   \label{fig:u2_u0}
		\end{subfigure}
		\hfill
  \begin{subfigure}[b]{0.326\textwidth}
			\centering
			\input{ex-pdfa-learn/ex-d}
   \caption{ $\pnfa_0 / J(\pi_1, \{a\}, \{b\})$ [rejected merge].   $\pi_2 = \pi_1$.}
     \label{fig:u2_u1}
		\end{subfigure}
		\hfill
\begin{subfigure}[b]{0.326\textwidth}
			\centering
\input{ex-pdfa-learn/ex-e}			
   \caption{$\pnfa_0 / J(\pi_2, \{\epsilon\}, \{aa\})$ [accepted merge].}
     \label{fig:u3_u0}
		\end{subfigure}
		\hfill  
\begin{subfigure}[b]{0.326\textwidth}
			\centering
			 \input{ex-pdfa-learn/ex-f}
   \caption{Determinization, merging $\{ aab\}$ with $     \{ b\}$.}
     \label{fig:u7_u2}
		\end{subfigure}
    \begin{subfigure}[b]{0.326\textwidth}
			\centering
			\input{ex-pdfa-learn/ex-g}
   \caption{Determinization, merging  $  aabb$ with $  bb$, obtain $\pi_3$ and $\pi_4=\pi_3$.}
     \label{fig:u12_u6}
		\end{subfigure}
		\hfill
\begin{subfigure}[b]{0.326\textwidth}
			\centering
\input{ex-pdfa-learn/ex-h}			
   \caption{$ \pnfa_0 / J(\pi_4, \{ab\}, \{ba\}) $ and then merging $ab$ with $ba$, $baa$ with $aba$ due to determinization. }
     \label{fig:u5_u4}
		\end{subfigure}
		\hfill  
\begin{subfigure}[b]{0.326\textwidth}
			\centering
			\input{ex-pdfa-learn/ex-i}
   \caption{ $\pnfa_0 /  J(\pi_5, \{bb\}, \{\epsilon\})$ and then merging $bbb$ with $b$ due to determinization, obtain $\pi_6$.}
     \label{fig:u6_u0}
		\end{subfigure}
\begin{subfigure}[b]{0.326\textwidth}
			\centering
\input{ex-pdfa-learn/ex-j}			
   \caption{Merging $ aba$ with $ b$, $\pnfa_0 / J(\pi_6, \{aba  \},\{b\})$ and subsequently $abaa$ with $ab$.}
     \label{fig:u8_u2}
		\end{subfigure}  \quad 
\begin{subfigure}[b]{0.326\textwidth}
			\centering
\input{ex-pdfa-learn/ex-k}			
   \caption{Merging $ abb$ with $ a$ and then $abbb$ with $ab$ due to determinization.}
     \label{fig:u9_u1}
		\end{subfigure}    

        \caption{Steps of our algorithm for \ac{pnfa} inferring for the sample in Figure~\ref{fig:sample}. The partial order learned by our algorithm is shown in Figure~\ref{fig:learn_pref_graph}.}
  \label{fig:pref_tree_and_rpni}
	\end{figure*}

\end{example}


\section{Identification in the limit}
\label{sec:iden_in_limit}
Though Algorithm~\ref{alg:euclid}
learns a \ac{pdfa}, its correctness  in learning a \ac{pdfa} equivalent to the \ac{pdfa} from which the sample is generated,  is only ensured when the sample satisfies a specific condition. 
In this section, we define the condition of sample and prove that if a given sample satisfies this condition, then the proposed algorithm accurately learns the canonical \ac{pdfa} of the   preference model from the sample is drawn from. 

We first introduce two definitions, extending the shortest prefix and nucleus of a \ac{dfa}~\cite{oncina1992identifying} to those of a \ac{pdfa}.
\begin{definition}
\label{def:short_nucl}
Let $\pdfa= \langle Q, \Sigma, \delta, q_0, O, \succeq, \lambda  \rangle$ be the minimal \ac{pdfa} the sample $S$ is drawn from.
Given a state $q \in Q$, the \emph{shortest prefix} of $q$, denoted $\SPref(q)$, is the shortest word in the lexicographical ordering of $\Sigma^*$ that reaches state $q$, \ie, 
\begin{multline*}
\SPref(q)  \in \Sigma^\ast 
\text{ such that, } \delta (q_0, \SPref(q) )=q \text{ and } \\
\forall w' \in \Sigma^\ast,\delta (q_0, w')=q \implies w' >_L \SPref(q)   .
  \end{multline*}
Note that by construction $\SPref(q)$ is a singleton. Thus, we use  $\SPref(q)$ to refer to the shortest prefix for state $q$.
 The \emph{shortest prefixes of $\pdfa$}, denoted $\SPref(\pdfa)$, is the set of all the short prefixes of the states within $\pdfa$, i.e.,
\[
\SPref(\pdfa) =  \{\SPref(q) \mid q \in Q \}.
\]
\end{definition}
The \emph{nucleus} of $\pdfa$, denoted $\NU(\pdfa)$, is the set of all one-letter extensions of the words within $\SPref(\pdfa)$ and the empty string,
\[
\NU(\pdfa) = \{ ua \mid u \in \SPref(\pdfa) \text{ and } a \in \Sigma \} \cup \{\epsilon\}.
\]
While the shortest prefixes set contains the set of minimal words required to reach all the states of $\pdfa$, the nucleus contains a set of words that cover all the transitions in the automaton, including the transition on `empty string'   that reaches the initial state.

\begin{example} Consider the \ac{pdfa} in Figure~\ref{fig:pdfa}.
For this \ac{pdfa}, $\SPref(00) = \{ \epsilon \}$, $\SPref(10) = \{a\}$, $\SPref(01) = \{b\}$, $\SPref(11) = \{ab\}$.
Accordingly,
\[
\SPref(\pdfa) = \{\epsilon, a, b, ab \}
\]
and
\[
\NU(\pdfa) = \{\epsilon, a, b, aa, ab, ba, bb, aba, abb \}.
\]
\end{example}

\begin{definition}
        \label{def:closureS}
            Given a sample $S \subset \Sigma^* \times \Sigma^* \times \{0, 1, \perp\}$, the \emph{transitive closure} of $S$, denoted $S^\circ$, is the minimal
            subset  $S' \subset \Sigma^* \times \Sigma^* \times \{0, 1, \perp \}$ 
            such that 
            \begin{enumerate}
                \item $S \subseteq S'$;
                \item \label{cond:eq_reflexive} $\{(w, w, 0) \mid w \in W_S \} \subseteq S'$;
                \item \label{cond:eq_incomp_symmetric} for $b \in \{0, \perp\}$, if $(w, w', b) \in S'$, then $(w', w, b) \in S'$;
                \item \label{cond:eq_succ_transitive} for $b\in \{0,1\}$, if $(w, w', b) \in S'$ and $(w', w'', b) \in S'$, then $(w, w'', b) \in S'$;
                \item \label{cond:eq_succ_incomp} for $b \in \{1, \perp \}$, if $(w, w', 0) \in S'$ and $(w', w'', b) \in S'$, then $(w, w'', b) \in S'$; and
                \item \label{cond:succ_incomp_eq} for $b \in \{1, \perp\}$, if $(w, w', b) \in S'$ and $(w', w'', 0) \in S'$, then $(w, w'', b) \in S'$.
            \end{enumerate}
        \end{definition}
        %

    Notice that the size of $S^\circ$ is polynomial to the size of $S$. 
    More precisely, $|S^\circ| = \bigO(|S|^2)$. 
    This is because between any two words $w, w' \in W_S$, at most one tuple $(w, w', b)$ could exist in $S^\circ$, and that $|W_S| = O(|S|)$.
    The time required to construct $S^\circ$ is polynomial to the size of $S$.
    To construct $S^\circ$, one can extend common algorithms used for making the transitive closure of a relation.
    For example, we can extend the Floyed-Warshall algorithm~\cite{floyd1962algorithm}, by which, the time complexity will be $\bigO(|W_S|^3)=\bigO(|S|^3)$.
    
    These above definitions are used to identify a family of samples for which our algorithm can efficiently learn the canonical \ac{pdfa} the sample is drawn from.
    %
\begin{definition}
\label{def:char_set}
Let $\pdfa= \langle Q, \Sigma, \delta, q_0, O, \succeq, \lambda \rangle$ be the minimal \ac{pdfa} where a sample $S$ is drawn from.  
This sample  $S$ is \emph{characteristic} if
\begin{enumerate}
    \item \label{itm:char_nu_subset} $\NU(\pdfa) \subseteq \Pref(W_S)$;
    \item \label{itm:char_sp_and_nu} for each $w \in \SPref(\pdfa)$ and $u \in \NU(\pdfa)$, if $\delta (q_0, w) \neq \delta (q_0, u)$, then there exists $y \in \Sigma^*$ such that
        $wy \in W_S$ and $uy \in W_S$, and $$\{(wy, uy, 1), (uy, wy, 1), (wy, uy, \perp)\} \cap S^\circ \neq \emptyset;$$
    \item \label{itm:char_pref_graph} For each $q \in Q$, there exists $w \in W_S$ such that $\delta(q_0, w) = q$;  and
    \item \label{itm:char_pref} for each $w, u \in W_S$, 
    \begin{enumerate}
        \item if $w \sim u$, then $(w, u, 0) \in S^\circ$,
        \item if $w \succ u$, then $(w, u, 1) \in S^\circ$,
        \item if $w \nparallel u$, then $\{(w, u, 0), (w, u, 1), (u, w, 1)\} \cap S^\circ = \emptyset$.
    \end{enumerate}
\end{enumerate}
\end{definition}
Condition~\ref{itm:char_nu_subset} indicates that any  transition of the canonical \ac{pdfa} can be   triggered by reading at least one prefixes of the sample.
Condition~\ref{itm:char_sp_and_nu} guarantees that the states corresponding to those two words $w$ and $u$ in the prefix tree automaton are not merged by the algorithm (see  the proofs next). 
Conditions~\ref{itm:char_pref_graph} ensures that each rank is represented by at least one word within the sample. Condition~\ref{itm:char_pref}    ensures the partial order over the set of ranks can be inferred correctly (see the proofs next). 
%

Next, we present our main results about our algorithm.
The summary of these results is as follows:
\begin{itemize}
    \item In Lemma~\ref{lem:not_merged}, we establish a condition that prevents the merging of two blocks with inconsistent ranks. This lemma is used in the proof of Lemma~\ref{lem:learned_pdfa_isomorphic}.
    \item In Lemma~\ref{lem:state_mapping}, we show if the sample is characteristic, then there is a one-to-one mapping between the states of the learned \ac{pnfa} and the canonical \ac{pdfa} from which the sample is taken.  
    \item In Lemma~\ref{lem:deterministic_trans_func}, we prove that if the sample is characteristic, then the learned \ac{pnfa} is in fact a \ac{pdfa}.
    \item In Lemma~\ref{lem:learned_pdfa_isomorphic}, we show that if the sample is characteristic, then the one-to-one mapping between states of the learned \ac{pdfa} and the canonical \ac{pdfa} produces an isomorphism between the learned \ac{pdfa} and the canonical \ac{pdfa}, and hence, our algorithm accurately learns the canonical \ac{pdfa}.
\end{itemize}

  First, we establish conditions under which two blocks are prevented from being merged.
\begin{lemma}
\label{lem:not_merged}
Consider a ranking-consistent partition $\pi$ of $W_S$, two blocks $B_1, B_2 \in \pi$ and two words $u_1 \in B_1$ and $u_2 \in B_2$ selected from two blocks. 
If there exists $y \in \Sigma^*$ such that one of the  tuples  $(u_1y,  u_2y, 1)$ or  $(u_2 y, u_1y, 1)$ or $(u_1y, u_2y, \perp)$ is included in $S^o$, then $D(\pdfa_0 / J(\pi, B_1, B_2)) = \undefined$, meaning that
  $B_1$ is not merged with $B_2$ by Algorithm~\ref{alg:euclid}.
\end{lemma}
\begin{proof}
    Let $B_3$ and $B_4$ be the blocks of $\pi$ that contain $u_1y$ and $u_2y$, respectively, and let $y=a_1 a_2 \cdots a_k$.
    Given that $S^\circ$ contains at least one of the tuples $(u_1y,  u_2y, 1)$, $(u_2 y, u_1y, 1)$, and $(u_1y, u_2y, \perp)$, 
    states $u_1y$ and $u_2y$ cannot belong to the same SCC of the graph $G_{\sim}$ and  in $\pdfa_0$, and thus $\lambda_{S}(u_1y) \neq \lambda_{S}(u_2y)$. 
    This  implies that $B_3$ and $B_4$ are assigned different ranks, and hence, they cannot be merged given Case 2 of the determinization operation.

    By way of contradiction, let's assume that $B_1$ is merged with $B_2$. 
    Result of this merging is the automaton $\pdfa_0 / J(\pi, B_1, B_2)$, which is not determinized yet.
    Figure~\ref{fig:inconsistent_merging} depicts this automaton. 
    As both $B_{11}$ and $B_{21}$ are $a_1$-successors of state $B_1 \cup B_2$, the determinization of $\pdfa_0 / J(\pi, B_1, B_2)$, requires $B_{11}$ to be merged with $B_{21}$, which forms a new automaton 
    $\pdfa_0 / J(J(\pi, B_1, B_2), B_{11}, B_{21})$.
    Repeating this argument, for all $1 \leq j \leq k$,  $B_{1j}$ has to be merged with $B_{2j}$ in the determinization    $D(\pdfa_0 / J(\pi, B_1, B_2))$.
    %

    %
    Because in this automaton, $B_{3} = B_{1k}$ and $B_{4} = B_{2k}$ are  $a_k$-successors of $B_{1(k-1)} $ and $ B_{2(k-1)}$, respectively, the determinzation requires them to be merged. However, because they have
    inconsistent ranks, they cannot be merged, and hence, $D(\pdfa_0 / J(\pi, B_1, B_2)) = \undefined$, meaning that our algorithm rejects merging $B_1$ with $B_2$, which is a contradiction.
   On the other hand, if there exists  $j < k$, $B_{1j}$ and $B_{2j}$ have inconsistent ranks, then  $B_{1j}$ cannot be merged with $B_{2j}$ either, which is again a contradiction.

Hence, $B_1 $ and $B_2$ cannot be merged.
\end{proof}
 \begin{figure}[t]
		\centering
		\begin{subfigure}[b]{0.27\linewidth}
			\centering
			\includegraphics[scale=0.44]{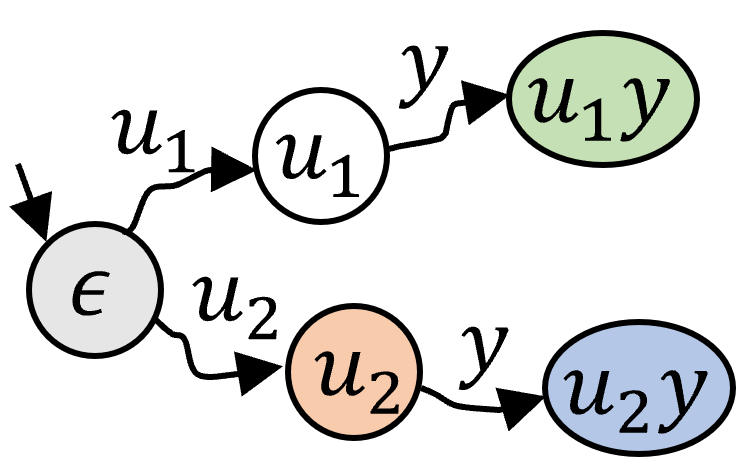}
   \caption{$\pnfa_0$.}
   \label{fig:A0}
		\end{subfigure}
		\hfill
  \centering
		\begin{subfigure}[b]{0.63\linewidth}
			\centering
			\includegraphics[scale=0.44]{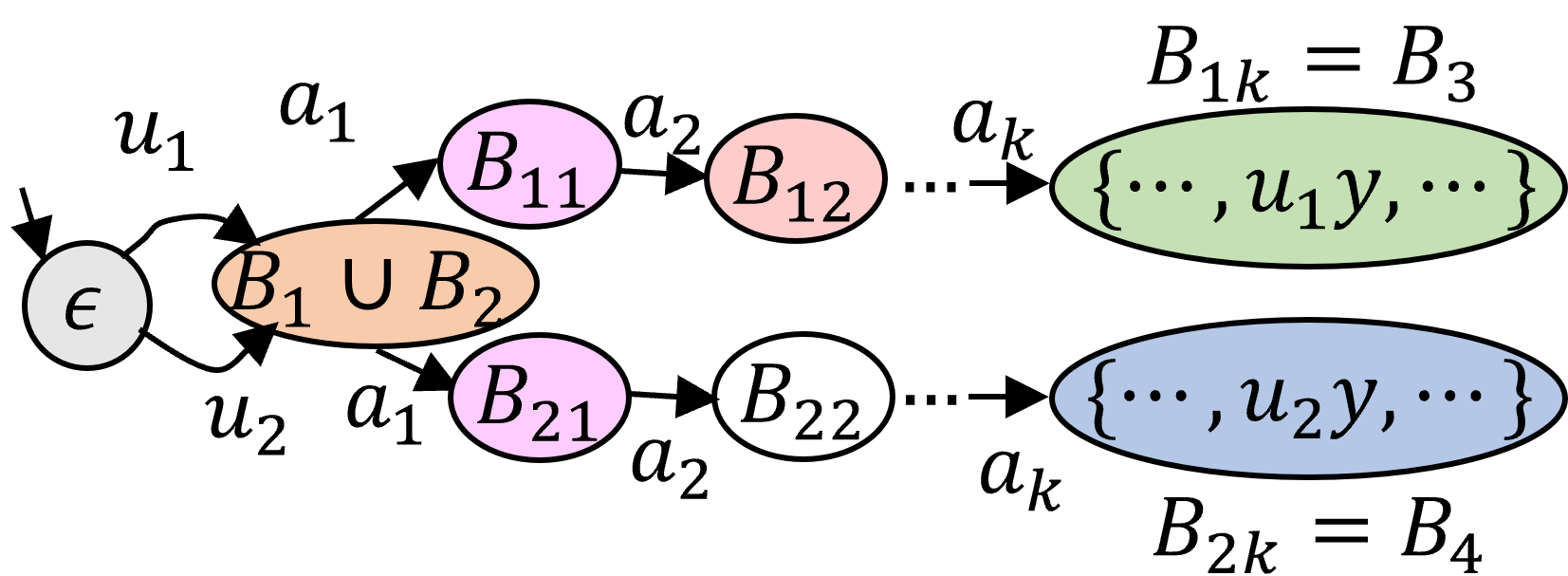}
   \caption{$\pnfa_0 / J(\pi, B_1, B_2)$}
     \label{fig:u1_u0}
		\end{subfigure}
		\hfill

        \caption{It illustrates why if two blocks that contains two words with inconsistent extensions cannot be merged.}
  \label{fig:inconsistent_merging}
	\end{figure}
The following lemma shows that there is a bijection between the state space of the learned \ac{pnfa} and the state space of the canonical \ac{pdfa}. Recall that $\eta(U)=\min_{<_L}\{w \in U\}$.

In the following results, 
  $S$ refers to  a characteristic sample drawn from a canonical \ac{pdfa}  $\pdfa= \langle Q, \Sigma, \delta, q_0, O, \succeq, \lambda \rangle$, and
 $\pi_r$ is the partition obtained by Algorithm~\ref{alg:euclid}. The   \ac{pdfa} obtained at iteration $i$ is 
 \[
 \calA_i = \langle \pi_i, \Sigma, \delta_i, O_S, \succeq_S, \lambda_i \rangle.
 \]
 And
 the learned \ac{pdfa} is 
 \[
 \calA_r= \langle \pi_r, \Sigma, \delta_r, O_S, \succeq_S, \lambda_r \rangle
 \]
 where the set of ranks $O_S$ and partial order $\succeq_S$ over ranks are the same as in the prefix tree automaton $\calA_0$. 

\begin{lemma}
\label{lem:state_mapping}
 The following properties hold:
\begin{enumerate}
    \item \label{itm:sp_in_blocks} For any $w \in \SPref(\pdfa)$, there exists $B \in \pi_r$ such that $w \in B$.    
    \item \label{itm:sp_to_block} For each $q \in Q$, there is a block $B \in \pi_r$ such that $\SPref(q) = \eta(B)$.
    \item \label{itm:block_to_sp} For each $B \in \pi_r$, there is a state $q\in Q$  such that $q = \delta (q_0, \eta(B))$ and   
    $\eta(B) = \SPref(q)$. 
\end{enumerate}
\end{lemma}

\begin{proof} 
    (\ref{itm:sp_in_blocks}) Because $S$ is characteristic, by Condition~\ref{itm:char_nu_subset} of Definition~\ref{def:char_set}, $\NU(\pdfa) \subseteq \Pref(W_S)$. This implies $\SPref(\pdfa) \subseteq \Pref(W_S)$. Accordingly, for any shortest prefix $w\in \SPref(\pdfa)$, $w$ is a state of the prefix tree automaton $\PTree_S$, which means $\{w\}$ is a block of $\pi_0$. As a result, in each step $i$ of the algorithm, there is a block in $\pi_i$ that contains $w$. 

   (\ref{itm:sp_to_block}) 
   \begin{figure}
       \centering
\includegraphics[width=0.8\linewidth]{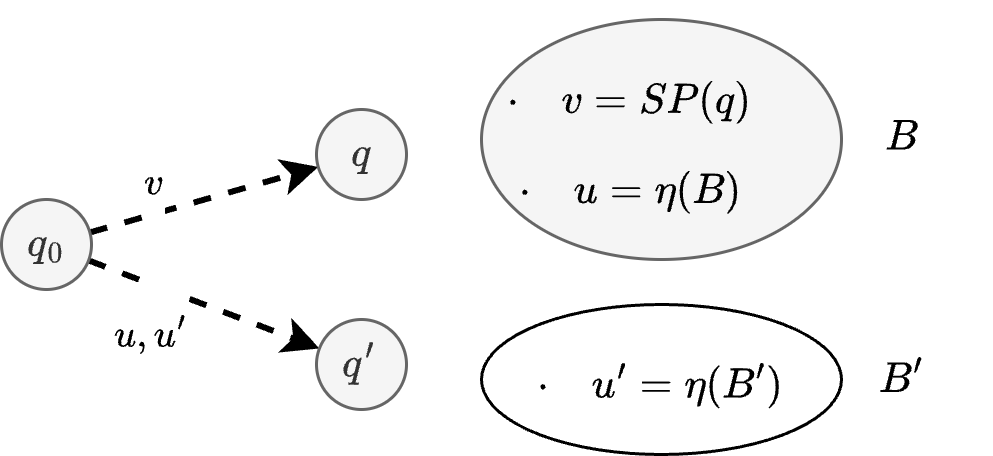}
       \caption{Illustration two blocks $B$ and $B'$ and the contained prefixes $u,v, u'$. }
       \label{fig:lem3-2}
   \end{figure}
  According to \eqref{itm:sp_in_blocks},  there exists a block $B\in \pi_r$ such that $\SPref(q) \in B$. By way of contradiction, assume that $\eta(B) \ne \SPref(q)$. That is, let $u= \eta(B)$ and $v= \SPref(q)$, it is the case that  $u\ne v$, and $u<_L v$ (by definition of $\eta$.)

Suppose $q' = \delta(q_0, u)$ in the canonical \ac{pdfa}  $\pdfa$. Clearly, $q'\ne q$ because if $q'=q$, then $v = \SPref(q)$ implies $v <_L u$, while the above analysis showed that $u<_L v$. 

Let's introduce another prefix $u' =\SPref(q')$. By \eqref{itm:sp_in_blocks}, there exists $B'$ such that $u'\in B'$. See the illustrated relation in Figure~\ref{fig:lem3-2}.
Next,   by the definition of characteristic sample, there exists $y\in \Sigma^\ast$ such that $u'y\in W_S$ and $vy \in W_S$, and 
    $$\{( u'y, vy, 1), (u'y, vy, 1), (u'y, vy, \perp)\} \cap S^\circ \neq \emptyset;$$

    Then, the block $B'$ containing $u'$ shall not be merged with the block $B$ containing $v$ by Lemma~\ref{lem:not_merged}.
    
    However, when attempting to merge $B$ and $B'$  in a deterministic  quotient automaton $\pdfa_i $, this merge operation will be rejected only if there exists $y \in \Sigma^\ast$ such that $\lambda_S (\delta_i(\epsilon, uy)) \ne \varnothing$ and $\lambda_S (\delta_i(\epsilon, u'y) ) \ne \varnothing$ and $\lambda_S (\delta_i(\epsilon, uy)) \ne  \lambda_S (\delta_i(\epsilon, u'y)) $, which is not possible because $\delta(q_0,u)=\delta(q_0,u')$ in the original \ac{pdfa}. For any $w\in \Sigma^\ast$, $b\in \{0, 1,\bot\}$, $(uy,w,b)$ if and only if $(u'y,w,b)$. Therefore, the two blocks $B'$ containing $u'$ and $B$ containing $u$  has to be merged --- a contradiction to the previous statement that $B$ and $B'$ cannot be merged since $u'=\SPref(q)$ and $v=\SPref(q)$.
    
    As a result, it is  only possible that  $u=\eta(B)=v =\SPref(q)$. 

\eqref{itm:block_to_sp} 
\textcolor{black}{To ease the proof, we assume the following implementation of Algorithm~\ref{alg:euclid}: After each merging, it will decide to whether reject the merge or not based on the determinization result. However, at iteration $i$, for any $B>_L B_i$, the singleton block will be kept in the intermediate non-deterministic automaton. The automaton is determinizied at the last iteration only.} \footnote{This is a similar procedure used in the proof of state-merging algorithm \cite{oncina1992inferring,oncina1992identifying}. The determinization is only carried out at the last step but whether the merged automaton can be determinized without raising any contradiction to the data is checked for every step. This merging-and-last-step-determinization procedure has been showin in the original method to be equivalent to the determinized merging procedure and is mainly used for the ease of proof.}

  Due to the completeness of transition function in \ac{pdfa} $\pdfa$, for each $B\in \pi_r$, there exists a state $q\in Q$ such that $q=\delta(q_0, \eta(B))$.

  For an arbitrary pair $B$ and $q=\delta(q_0, \eta(B))$,
let $v=\SPref(q)$ and $u = \eta(B)$. By way of contradiction, suppose $u\ne v$. Then, by the fact that $v=\SPref(q)$ and $\delta(q_0, u)=\delta(q_0,v)=q$, $v<_L u$.
  By the definition of a characteristic sample, both $u, v \in W_S$. A block containing $v$, say $B_v$, shall be considered to merge with $\{u\}$ given $v <_L u$, during one of the intermediate steps. This merge operation  is only rejected when there exists a word 
  %
$y \in \Sigma^\ast$ and a word $w\in \Sigma^\ast$ such that one of the following cases occurs:
\begin{itemize}
\item  $\{(uy, w, 1), (w, vy, b)\} \cap S^o\ne \emptyset$, for $b\in \{1,\bot, 0\}$,
\item $\{( w, uy, b), ( vy, w, 1)\} \cap S^o \ne \emptyset$, for $b\in \{0,\bot, 1\}$,
\end{itemize} 
This is because only if  one of these cases occurs, then $\lambda_S(\delta_i(B_\epsilon), uy) \ne \lambda_S(\delta_i(B_\epsilon), vy)  $ for any iteration $i =1,\ldots, n$, and $B_\epsilon$ is the block that contains the empty string $\epsilon$.

However, none of the   cases is possible because $\delta(q_0, u)=\delta(q_0, v)$, which implies for any suffix $y$ and $w\in \Sigma^\ast$, $(uy, w, b)$ if and only if $  (vy, w, b)$, for $b\in\{0, 1,\bot\}$. As a result, the merge cannot be rejected and $u$ is being merged with the same block that contains $v$. And $v <_L u$ implies $v = \eta(B)$. Thus, both $u=v = \eta(B)$---a contradiction to the assumption that $u\ne v$. Therefore, we conclude that $u= \eta(B)=v=\SPref(q)$ for $q=\delta(q_0, \eta(B))$, for any $B\in \pi_r$. 

   \end{proof}

Lemma~\ref{lem:state_mapping} provides the bijection function: $$g: \pi_r \rightarrow Q\text{ such that }\forall  B \in \pi_r, g(B) = \delta(q_0, \eta(B)).$$ 

So far, we have proven that if $S$ is a characteristic sample, then the proposed algorithm accurately learns the state space of the 
canonical \ac{pdfa} $\pdfa$.  Next, we     show a  bijection   can be established between the transitions in the learned \ac{pnfa} and the transitions  in the canonical \ac{pdfa}.
%
%

 \begin{figure}[t]
		\centering
		\begin{subfigure}[b]{0.49\linewidth}
			\centering
			\includegraphics[scale=0.4]{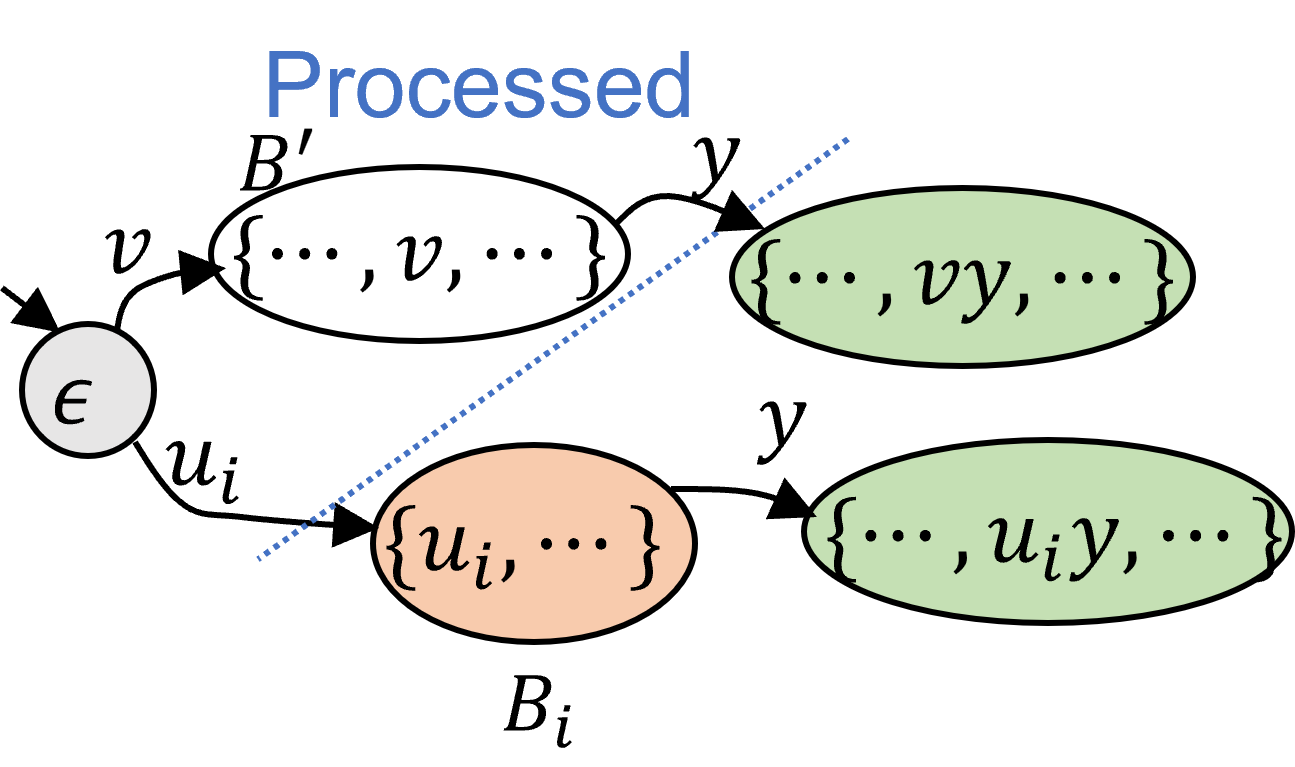}
   \caption{$\pnfa_0 / \pi_{i-1}$.}
   \label{fig:A0_pi_i_1}
		\end{subfigure}
		\hfill
  \centering
		\begin{subfigure}[b]{0.49\linewidth}
			\centering
			\includegraphics[scale=0.4]{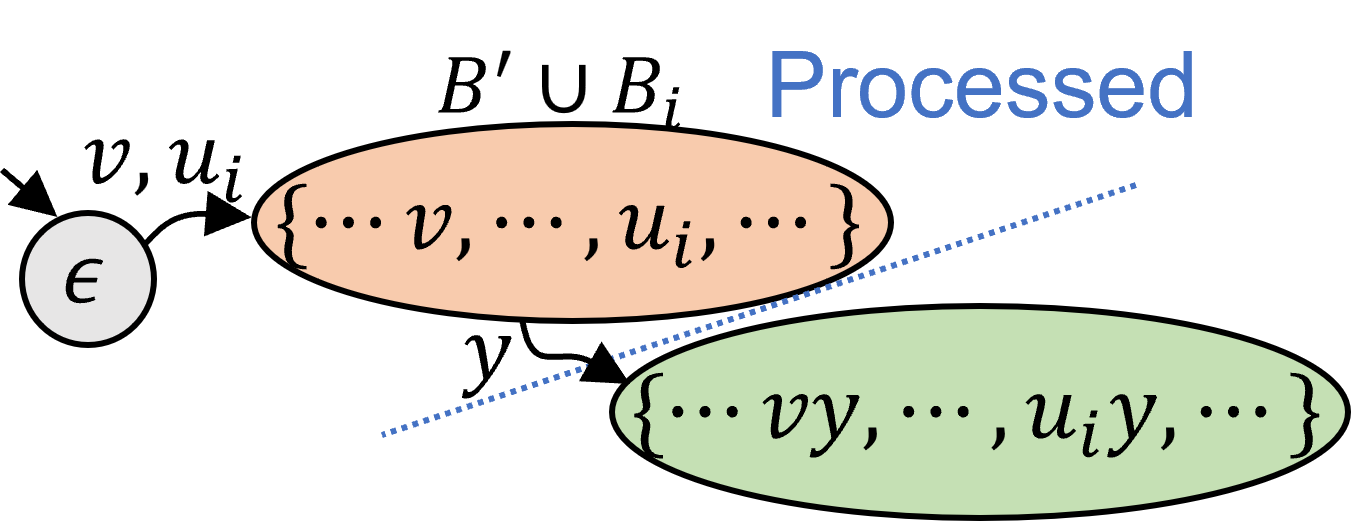}
   \caption{$\pnfa_0 / J(\pi, B', B_i)$}
     \label{fig:B1_B2_merged}
		\end{subfigure}
		\hfill

        \caption{
        \textbf{a)} It shows that at the $i$-th iteration, the algorithm processes $B_i$ to merge with a block $B' <_L B$. It holds that $u_i = \eta(B_i)$ and that $\delta(q_0, u_i) = \delta(q_0, v)$. The part that appears before the dashed line is the processed part of the automaton. For any block $B''$ in this part, it holds that $B'' <_L B_i$.
        \textbf{b)} Blocks $B_i$ and $B'$ are merged because $\delta(q_0, u_i) = \delta(q_0, v)$. Also, since $v <_L u_i$ and that $\{v, u_i\} \subseteq B_i \cup B'$, $u_i \neq \eta(B_i \cup B')$.}
  \label{fig:merging_with_sp}
	\end{figure}

We first show that the learned automaton is deterministic and its transition function is complete. 

\begin{lemma}
\label{lem:deterministic_trans_func}
%
For any $B \in \pi_r$ and $a \in \Sigma$, 
$|\delta_r(B, a)| = 1$.
\end{lemma}
\begin{proof}
 First, given that at each step, the algorithm uses determinization after merging, it holds that $|\delta_r(B, a)| \leq 1$. When $|\delta_r(B, a)|  =0$, $\delta_r(B, a) $ is undefined.

Thus, we only need to show that for any $B\in \pi_r$ and $a\in \Sigma$, $  \delta_r(B, a) $ is defined.
%
Firstly, state $B$ is reachable by word $\eta(B)$, and by Lemma~\ref{lem:state_mapping}, $\eta(B) \in \SPref(\pdfa)$.
Thus, $\eta(B)a \in \NU(\pdfa)$.
Because $S$ is characteristic, by Condition (\ref{itm:char_nu_subset}) of Definition~\ref{def:char_set}, $\NU(\pdfa) \subseteq \Pref(W_S)$.
Hence, $\eta(B)a$ is a state in $\PTree_S$.
This implies that for every step $1 \leq i \leq r$, there exist two blocks $B_1, B_2 \in \pi_i$ such that $\eta(B) \in B_1$, $\eta(B)a \in B_2$, and there is a transition with $a$ from block $B_1$ to  $B_2$ given the definition of a quotient.   
This completes the proof.
 %
%
 %
 %
 \end{proof}
   
  The following two lemmas can be proven:
 \begin{lemma}
 \label{lem:learned_pdfa_isomorphic}
   The transitions in \ac{pdfa}s $\pdfa$ and $\pdfa_r$ are isomorphic:  For  any $B, B' \in \pi_r$ and $a \in A$,   $\delta_r(B, a)=B'$ \emph{if and only if} $\delta(g(B), a)=g(B')$. 
 \end{lemma}
 \begin{proof}
   As illustrated in Figure~\ref{fig:algorithm_automata_isomorphic},
  there are two cases to be considered: (1) $\eta(B)a \in \SPref(\pdfa)$, and (2) $\eta(B)a \notin \SPref(\pdfa)$.
    Case (1):  Since $\eta(B)a \in \SPref(\pdfa)$, by Lemma~\ref{lem:state_mapping}, $\eta(B') = \eta(B)a$. Hence,
     \begin{multline*} \delta(g(B), a)=\delta(\delta(q_0, \eta(B)), a)=\\
    \delta(q_0, \eta(B)a)=\delta(q_0, \eta(B'))=g(B').\end{multline*}
    %
  
  %
  %
  Case  (2): 
   Because $\eta(B)a \not\in \SPref(\pdfa)$, by Lemma~\ref{lem:state_mapping}, $\eta(B') \neq \eta(B)a$, which implies that block $B'$ contains two different words $\eta(B')$ and $\eta(B)a$. 
  
  By way of contradiction, assume $\delta(g(B), a)=q \neq g(B') $ given $g (B')=\delta(q_0, \eta(B'))$.
  %
  See Figure~\ref{fig:pdfa_isomorphic_contradiction} for illustration.
  Since $\eta(B) \in \SPref(\pdfa)$ (Lemma~\ref{lem:state_mapping}),   $\eta(B)a \in \NU(\pdfa)$. 
  Given that $S$ is characteristic, $\eta(B') \in \SPref(\pdfa)$. If $\delta(q_0, \eta(B')) \ne \delta(q_0, \eta(B)a)$, then  there exists a suffix $y$ such that either $(\eta(B')\cdot y, \eta(B)a\cdot y, 1) \in S^\circ$ or
  $(\eta(B)a \cdot y, \eta(B') \cdot y, 1) \in S^\circ$ or $(\eta(B') \cdot y, \eta(B)a \cdot y, \perp) \in S^\circ$. Thus, according to Lemma~\ref{lem:not_merged},
  $\eta(B')$ and $\eta(B)a$ cannot  belong to the same block of $\pi_r$,  which contradicts that  $\{\eta(B'), \eta(B)a\} \subseteq B' $ and $\eta(B') \ne \eta(B)a$. Hence, $\eta(B') = \eta(B) a$ even if $\eta(B)a \notin \SPref(\mathcal{A})$. As a result, $\delta(g(B),a) = \delta(\delta(q_0, \eta(B)),a) = \delta(q_0, \eta(B'))= g(B')$ also holds when $\eta(B)a\notin \SPref(\mathcal{A}).$
  \end{proof}
    \begin{lemma}
    \label{lma:reachable}
        Given a word $w \in \Pref(W_S)$, let $B_w \in \pi_r$ be the block that contains $w$, i.e., $w \in B_w$. It holds that $\delta_r(B_0, w) = B_w$, where $B_0$ is the block that contains $\epsilon$.
    \end{lemma}
    \begin{proof}
    The property follows from the prefix tree construction and the deterministic join operation, and thus is omitted.

        
    \end{proof}

  \begin{corollary}
  \label{col:reachable}
For any $B \in \pi_r$, for any $w\in B$, $\delta (q_0,w) = g(B)$.
  \end{corollary}
  \begin{proof}
This property follows from Lemma~\ref{lma:reachable} and the fact that $g(B_0)= q_0$ and $ \delta (q_0, w) = g(\delta_r(B_0,w)) = g(B)$ due to the isomorphic transitions between $\pdfa$ and $\pnfa_r$.
   \end{proof}

   After showing the isomorphic relations between the state sets and transitions between the true \ac{pdfa} $\calA$ and learned \ac{pdfa} $\calA_r$,  we only need to show that $\langle \pi_r, \succeq_S \rangle$ and $\langle O, \succeq \rangle$ are isomorphic. This is achieved by 
  constructing a mapping $h: \pi_r \rightarrow O$ and showing that this mapping induces the isomorphism.
   The construction uses the properties of $S$ being characteristic to show that for each rank $o \in O$, the sample contains a word that represents that rank, and also, for each pair of distinct ranks $o_1$ and $o_2$ such that $o_1 \succeq o_2$, the sample compares two words from which the preference between $o_1$ and $o_2$ imposed by $\succeq$, is induced. 



 \begin{lemma}
\label{lma:isomorphic-output} 
Under the assumption that $S$ is a characteristic sample,
the following property holds: $\lambda_r(B) \succ_S \lambda_r(B')$ if and only if 
$  \lambda (g(B))) \succ  \lambda (g(B'))$.
 
\end{lemma} 
\begin{proof}
 We first show that  $\lambda_r(B) \succ_S \lambda_r(B')$ implies $ \lambda(g(B)) \succ \lambda(g(B'))$.   By construction, $\lambda_r(B) \succ_S \lambda_r(B')$ if there exist $w \in B$ and $w'\in B'$ such that $(w,w',1)\in S^o$ (see the definition of the relation in \eqref{eq:relation-rank}). 
Because $(w,w',1)\in S^o$, let $q= \delta(q_0, w) $ and $q' =\delta(q_0,w')$, it holds that $\lambda(q)\succ \lambda(q')$. 

Thus, to show $\lambda_r(B) \succ_S \lambda_r(B')$ implies $ \lambda(g(B)) \succ \lambda(g(B'))$. We only need to show that $g(B)= q$ and $g(B')= q'$. Recall the definition of function $g$ is such that $g(B) = \delta(q_0,\eta(B))$ where $\eta(B) = \min_{<L}\{ u \in B\}$. According to Lemma~\ref{lma:reachable} and Corollary~\ref{col:reachable}, for any $w\in B$,  $\delta(q_0, \eta(B)) = \delta(q_0, w)$. Thus, $q= \delta(q_0,w) =  \delta(q_0, \eta(B)) = g(B)$. The similar reasoning allows us to show $q' = \delta(q_0,w')=g(B')$.   

To show $\lambda(g(B)) \succ\lambda(g(B'))$ implies $\lambda_r(B) \succ_S \lambda_r(B')$, it is observed that if  $\lambda(g(B)) \succ\lambda(g(B'))$, then there must exists $(w,w',1)\in S^o$. Thus $\lambda_i (B(w,\pi_i))\succ_S  \lambda_i(B(w',\pi_i))$ is satisfied for all intermediate partitions $\pi_i$, $i=1,\ldots, n$ in Algorithm~\ref{alg:euclid}. As a result, $\pi_n = \pi_r$ is the partition obtained upon convergence. Thereby $B = B(w,\pi_r) $,  $ B(w',\pi_r)=B'$, and $\lambda_r (B(w,\pi_i))\succ_S  \lambda_r (B(w',\pi_i))$ is derived.
\end{proof}


\begin{figure}[t]
		\centering
		\includegraphics[width=\linewidth]{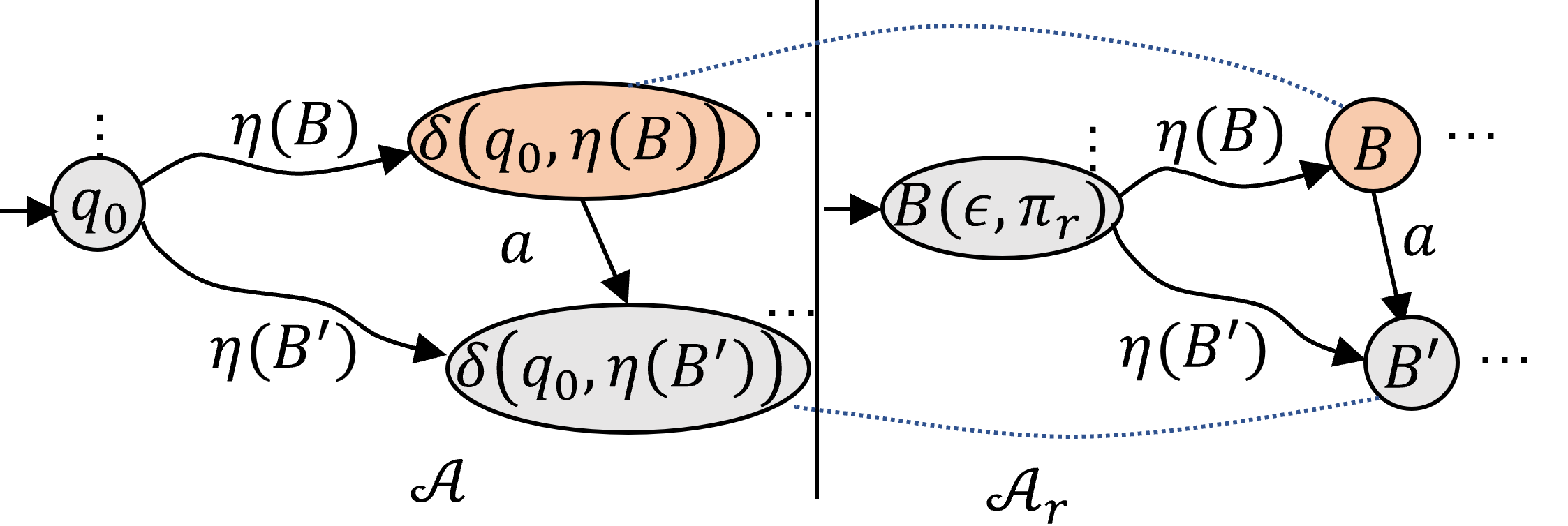}
		\caption{
		      The isomorphism between states of the learned \ac{pdfa} by our algorithm and the states of the canonical \ac{pdfa} the characteristic sample is drawn from. Words $B$ and $B'$ belong to the shortest prefixes of the canonical \ac{pdfa}. 
.		}
		\label{fig:algorithm_automata_isomorphic}
	\end{figure}

\begin{figure}[t]
		\centering
		\includegraphics[width=\linewidth]{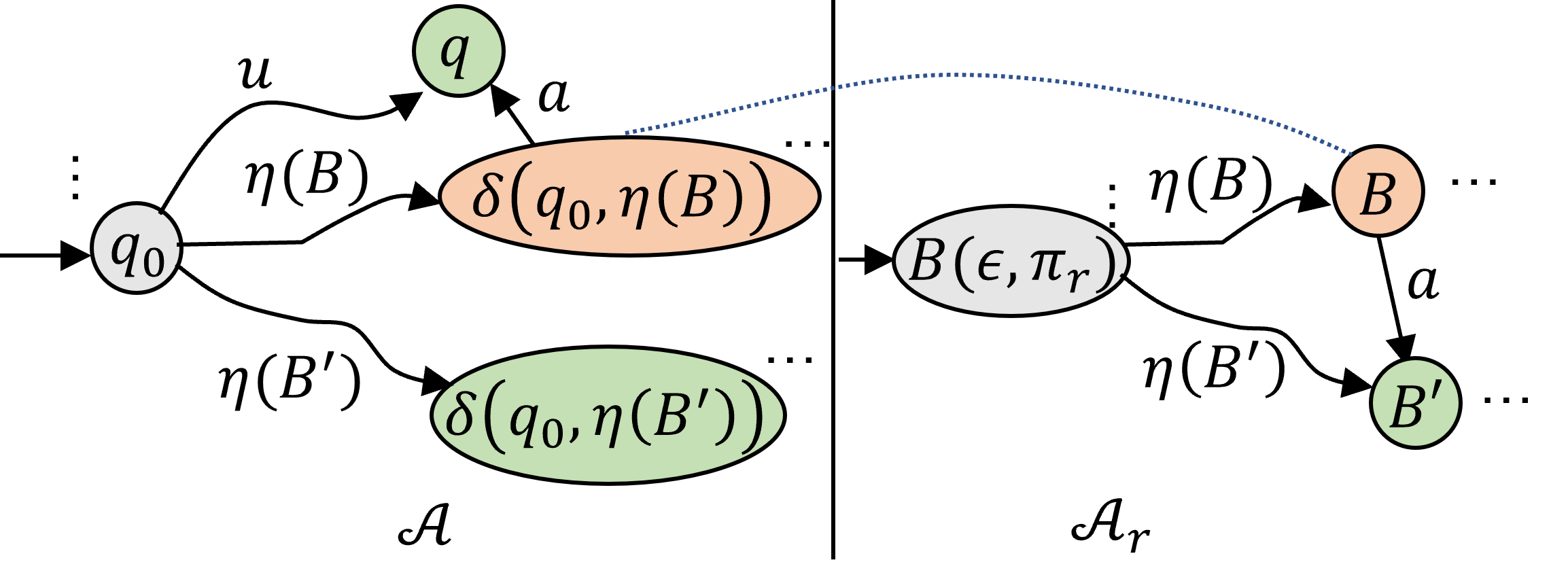}
		\caption{
		      The contradiction that $B'$ in $\pdfa_r$ is not mapped to $\delta(q_0, \eta(B'))$ in $\pdfa_0$ in the proof of Lemma~\ref{lem:learned_pdfa_isomorphic}.
.		}
		\label{fig:pdfa_isomorphic_contradiction}
	\end{figure}
 
\begin{theorem}
\label{thr:learn_pdfa}
If the sample $S$ in the \LPM is characteristic, then our algorithm computes the canonical presentation of the \ac{pdfa} from which $S$ is drawn in time 
$\bigO(m^3|\Sigma|^2+|S|^3)$, in which   $m =\sum_{w \in W_S} |w|$.
\end{theorem}
\begin{proof} 
    Lemma~\ref{lem:state_mapping} 
    establishes the bijection between two state spaces, Lemma    ~\ref{lem:learned_pdfa_isomorphic} establishes the bisimilation between two transition sets; and lemma \ref{lma:isomorphic-output} shows the bisimilation between the output ranking relations. With all three, the learned \ac{pnfa} is equivalent to the canonical \ac{pdfa} from which $S$ is drawn.

  The following analysis shows the time complexity of our algorithm.
    Constructing the graph $G_{~}$ takes $\bigO(|S|)$. From this graph it takes $\bigO(|G_\sim|^3)=\bigO(|S|^3)$ to construct $\rankS$ and the sample partial order $\succeq_S$.
    The number of stats within $\pdfa_0 = \PTree(S)$ is $\bigO(m)$, and hence, the time to construct $\pdfa_0$ is $\bigO(m|\Sigma|)$.
    The RPNI extension algorithm iterates $m$ times in the worst case and each iteration $u_i$ is checked with $i-1$ states $B < u_i$ to see if it can be merged with them.
    Checking wether the merging is allowed and accepted requires the determination process, which in the worst case takes $\bigO(m|\Sigma|)$.
    Hence, the running time of our algorithm is $\bigO(m^3|\Sigma|^2+|S|^3)$.
\end{proof}
Thus, our algorithm learns an accurate preference model if the given sample is characteristic.
%

We consider the following variant of our decision problem.
\decproblembox{Minimum Consistent PDFA with Characteristic Sample (\MCPDFAChar)}
        {A sample $S$ that is characteristic for the canonical representation of the preference model from which $S$ is taken, and a positive integer $k$.}
        {\Yes if there exists a \ac{pdfa} with at most $k$ states that is consistent with $S$, and \No otherwise.}
\begin{corollary}
    \MCPDFAChar $\in \P$.
\end{corollary}

\section{Case Studies}
\label{sec:case_studies}
In this section, we report results of our implementation.
We implemented our algorithm in Python and performed our experiments on a Windows 11 installed on a device with a core i$7$, 2.80GHz CPU and a 16GB memory. 
\subsection{Running Example}
We fed the sample in Figure~\ref{fig:sample} to the implementation of our algorithm and let it to learn the \ac{pdfa}.
The average running time over $10$ executions was $0.074$ seconds.
Figure shows the \ac{pdfa} learned by the implementation of our algorithm.
This \ac{pdfa} is equivalent to the same \ac{pdfa}s in Figure~\ref{fig:pdfa} and Figure~\ref{fig:u9_u1}, and provides evidence to the 
accuracy of our algorithm.

\begin{figure}[h]
\centering
\begin{subfigure}[b]{0.25\textwidth}
\input{figs/running-example-pdfa}
\end{subfigure}
\begin{subfigure}[b]{0.2\textwidth}
\input{figs/running-example-prefgraph} 
\end{subfigure}
\caption{The PDFA learned from the sample in Figure~\ref{fig:sample}.}\label{fig:casestudy_1_learned_pdfa}
\end{figure}
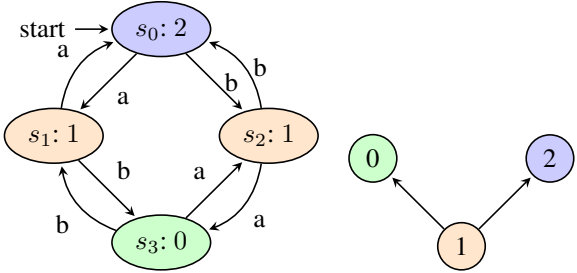 
            This \ac{pdfa} is isomorphic to the running example \ac{pdfa}, the one in Figure~\ref{fig:pdfa}.
            Under this isomorphism, states $s_0$, $s_1$, $s_2$ and $s_3$ of this \ac{pdfa} are respectively mapped to states $00$, $10$, $01$, and $11$ of the \ac{pdfa} in Figure~\ref{fig:pdfa}. It also maps the outputs $0$, $1$, and $2$ of the \ac{pdfa} in this figure respectively to outputs `green', `orange', and `blue' of the \ac{pdfa} in Figure~\ref{fig:pdfa}.

\subsection{Experiment validation: Bee Robot Motion Planning}
In this example, we provide results of our algorithm for a robotic planning problem described in Section~\ref{sec:intro}.

 In our implementation, we generated samples of words   with different sizes, $50$, $100$, $200$, ..., and $700$, as shown in Table~\ref{fig:table_1}. 
 For each number of words, we generated 10 samples of the same size.
 The generation of the words for a sample followed a random process where starting from an empty string, a word is extended by letting each of the letter $t$, $d$, and $o$ to have a probability of $0.25$ to be appended to that word and with probability $0.25$ it was opted to not extending the word.

 We did two experiments to generate comparisons between words.
 In the first experiment, each word of the sample is compared with one third of the words within the sample, while in the second experiment, each word of the sample is compared with half of the words within the sample.
 After generating the set, we verified if the sample is characteristic or not, using the condition in Def.~\ref{def:char_set}.

 Table~\ref{fig:table_1} shows results of the first experiment.
 The first column shows the number of words within the sample.
 The second column shows for each row, how many of the $10$ randomly samples were characteristic.
 According to this table, none of the samples with $50$ words were characteristic, and $8$ of the samples with $700$ words were characteristic. 
 The table also provides, for each row and for each of the four conditions outlined in Definition~\ref{def:char_set} for a sample to be characteristic,   the frequency of violations for each condition.
 %

 The $7$'th column, titled Canonical, indicates for each row, how many times the algorithm learned an equivalent canonical \ac{pdfa}.
The results validated that if a sample is characteristic, then the learned \ac{pdfa} is equivalent to the true \ac{pdfa}, proving the correctness of our algorithm. 
 Additionally, even for some samples that were not characteristic, the algorithm was still able to learn the correct \ac{pdfa}.
 For example, out of the $10$ samples with $100$ words each, none were characteristic, but the algorithm was able to learn the true canonical \ac{pdfa} for $3$ of them.

 For the second experiment, we used the same processes of generating the samples, but this time each word of the sample
 is compared with half of the words.
 Results of this experiment is shown in Table~\ref{fig:table_2}.
 While in the previous experiment $8$ of the $10$ randomly generated samples with $800$ words were characteristic, in this experiment, each of the scenarios for $200$, $400$, $500$, $600$, and $700$ words had more than $8$ characteristic samples.
 This proves that increasing the number of comparisons between the words of a sample, increases the likelihood of the sample to be characteristic.

\begin{table}[t]   
    \setlength{\tabcolsep}{4pt}
\renewcommand{\arraystretch}{1.5}
        \centering
        \small
        \begin{tabular}{lccccccc}
            \toprule
            \multicolumn{1}{l}{Words}   &\multicolumn{1}{c}{Characteristic}    &\multicolumn{4}{c}{Violated Condition}  &\multicolumn{1}{c}{Canonical} &\multicolumn{1}{c}{Time}  \\ 
            \cmidrule(lr){3-6} 
                                   &    & {1} & {2} & {3} & {4} &  & \\
            \midrule

            50 	&0 	&8 &0 	&0 	&4 	&1  &	0.047 \\
100  &	0 &	3 &	10 &	0 &	5 &	3 &	0.200 \\
200 & 	3 &	0 &	4 &	0 &	4 &	6 &	0.551 \\
300 & 	3 &	0 &	0 &	0 &	7 &	3 &	2.169 \\
400 & 	4 &	0 &	0 &	0 &	6 &	4 &	4.688 \\
500 & 	4 &	0 &	0 &	0 &	6 &	4 &	10.318 \\
600 & 	8 &	0 &	0 &	0 &	2 &	8 &	18.641 \\
700 & 	8 &	0 &	0 &	0 &	2 &	9 &	29.296 \\


            \bottomrule
        \end{tabular}
     \caption{Results of our implementation for learning the \ac{pdfa} in Figure~\ref{fig:garden_pdfa}.
        Each row shows results for $10$ randomly generated samples. 
        Each word is compared against one third of the words randomly. 
        }
        \label{fig:table_1}
        \end{table}

            \begin{table}[t] 
    \setlength{\tabcolsep}{4pt}
\renewcommand{\arraystretch}{1.5}
        \centering
        \small
         \begin{tabular}{lccccccc} 
            \toprule
            \multicolumn{1}{l}{Words}   &\multicolumn{1}{c}{Characteristic}    &\multicolumn{4}{c}{Violated Condition}  &\multicolumn{1}{c}{Canonical} &\multicolumn{1}{c}{Time}  \\ 
            \cmidrule(lr){3-6} 
                                   &    & {1} & {2} & {3} & {4} &  & \\
            \midrule

            50 &	0 &	8 &	10 & 0 &	0 &	1 &	0.047 \\
100 &	0 &	2 & 	10 & 	0 &	3 &	3 & 	0.142 \\
200 &	9 &	0 &	0 &	0 &	1 &	9 &	0.695 \\
300 & 	7 &	0 & 	1 & 	0 &	3 &	7 &	3.316 \\
400 &	10 &	0 &	0 &	0 &	0 &	10 &	9.064 \\
500 &	9 &	0 &	0 &	0 &	1 &	9 &	15.832 \\
600 &	9 &	0 &	0 &	0 &	1 &	9 &	25.172 \\
700 &	10 &	0 &	0 &	0 &	0 &	10 &	40.362 \\
            \bottomrule
        \end{tabular}
       \caption{Results of our implementation for learning the \ac{pdfa} in Figure~\ref{fig:garden_pdfa}.
        Each row shows results for $10$ randomly generated samples. 
        Each word is compared against one half of the words randomly. 
        }
        \label{fig:table_2}
        \end{table}

\section{Conclusions and Future Work}
\label{sec:conc}
In this paper, we consider the problem of  learning a preference deterministic finite automaton, representing  a user preference over a set of temporally extended goals. The learner is presented with a sample containing a collection of pairwise comparisons between temporal sequences/words (trajectories).
We demonstrate that this problem is computationally challenging and propose an inference algorithm based on a modification of the RPNI automata learning algorithm \cite{oncina1992identifying}. The proposed algorithm incorporates similar state-merging operations as RPNI but uses distinct decision rules to determine whether a merge is accepted or not.

We define the concept of characteristic samples and prove that, given a characteristic sample for the preference automaton, our algorithm successfully learns the minimal \ac{pdfa} equivalent to the true \ac{pdfa} from which the sample is drawn. The correctness and effectiveness of our algorithm are validated through several case studies, including scenarios with both characteristic and non-characteristic samples.

%
Future work will explore an active learning of preference,  where the learner can query a limited number of words and solicit human feedback on pairwise comparisons of specifically selected pairs based on the current dataset. In this context, the learner must strategically select pairs to minimize the number of queries required from the human. It is noted that our proposed algorithm is a passive, batch learner: If a new datum is added, then the learner must be restarted from scratch. 
Incremental learning algorithms could also be desired, where the learner receives data sequentially and incrementally updates the learned model with newly seen data. Additionally, future research will investigate integrating preference learning with preference-based decision-making in online planning setups. This approach aims to enable sample-efficient learning of control strategies directly from human feedback.

%


\section*{Acknowledgement}
The author(s) disclosed receipt of the following financial support for the research, authorship, and/or publication of this article: This work was supported by   the Air Force Office of Scientific Research under award number FA9550-21-1-0085 and in part by NSF under award numbers 2024802.

\bibliographystyle{plain}        
\bibliography{references}           



\end{document}

%% file: defs.tex
\newif\ifuseboldmathops
\newif\ifuseittextabbrevs
\useboldmathopstrue   

\ifuseittextabbrevs

	\newcommand{\ie}{{\it i.e.}}
	
	\newcommand{\etal}{{et~al.}}
\else

	\newcommand{\ie}{i.e.}
	
\fi

\ifuseboldmathops

\else

\fi

\ifuseboldmathops

\else

\fi

\ifuseboldmathops

\else

\fi

\ifuseboldmathops

\else

\fi

\newcommand{\undefined}{{\varnothing}}

\newcommand{\rankS}{O_S}

\newcommand{\abs}[1]{\lvert#1\rvert}

\newcommand{\bigO}{\mathcal{O}}

\renewcommand{\vec}[1]{\mathbf{#1}}

\acrodef{mdp}[MDP]{Markov Decision Process}
\acrodef{pomdp}[POMDP]{Partially Observable Markov Decision Process}
\acrodef{momdp}[MOMDP]{Multi-objective MDP}
\acrodef{ltl}[LTL]{Linear Temporal Logic}
\acrodef{ltlf}[LTL$_f$]{Linear Temporal Logic on Finite Traces}
\acrodef{dfa}[DFA]{Deterministic Finite Automaton}
\acrodef{nfa}[NFA]{Nondeterministic Finite Automaton}
\acrodef{nma}[NMA]{Nondeterministic Moore Automaton}
\acrodef{dma}[DMA]{Deterministic Moore Automaton}
\acrodef{tlmdp}[TLMDP]{Terminating Labeled Markov Decision Process}
\acrodef{pdfa}[PDFA]{Preference Deterministic Finite Automaton}
\acrodef{pnfa}[PNFA]{Preference Nondeterministic Finite Automaton}
\acrodef{ma}[MA]{Moore Automaton}

\theoremstyle{definition}

\ifdefined\begindefinition\else
\newtheorem{definition}{Definition}
\fi

\ifdefined\beginexample\else
 \newtheorem{example}{Example}
 \fi

 \ifdefined\beginproblem\else

\fi

\ifdefined\beginlemma\else
\newtheorem{lemma}{Lemma}
\fi

\ifdefined\beginassumption\else

\fi

\ifdefined\beginproposition\else
\newtheorem{proposition}{Proposition}
\fi

\ifdefined\begincorollary\else
\newtheorem{corollary}{Corollary}
\fi

\ifdefined\beginremark\else

\fi

\ifdefined\begintheorem\else
\newtheorem{theorem}{Theorem}
\fi

\ifdefined\beginproof\else
\newenvironment{proof}{\textbf{Proof:}}{\hfill$\square\\$}
\fi

\newcommand{\calA}{\mathcal{A}}

\acrodef{smdp}[Semi-MDP]{Semi-Markov decision process}
\acrodef{mcts}[MCTS]{Monte Carlo tree search}
\acrodef{uct}[UCT]{Upper Confidence Bound 1 applied to trees}
\acrodef{scltl}[scLTL]{syntactically co-safe LTL}
\acrodef{ssp}[SSP]{Stochastic Shortest Path}
\acrodef{p2sg}[SG(2)]{Two-player Stochastic Game}
\acrodef{mc}[MC]{Markov chain}
\acrodef{prefltl}[TPL]{ Temporal Preference Logic}
\acrodef{tld}[TLwD]{Temporal Logic with Distributions}
\acrodef{mtl}[Metric TL]{Metric Temporal Logic}
\acrodef{sta}[STA]{Stochastic Timed Automaton}

\acrodef{gpf}[GPF]{generalized preference formula}

\acrodef{cp}[CP]{ceteris paribus}
\acrodef{milp}[MILP]{Mixed-Integer Linear Programming}
\acrodef{dfa}[DFA]{Deterministic Finite Automaton}

\newcommand{\nfa}{\mathcal{A}}

\newcommand{\pdfa}{\mathcal{A}}
\newcommand{\pnfa}{\mathcal{A}}

\newcommand{\SPref}{\operatorname{SP}}
\newcommand{\NU}{\operatorname{NU}}

\newcommand{\LPM}{{\rm LPM}\xspace}

\newcommand{\MCPDFA}{{\rm MCPDFA}\xspace}
\newcommand{\MCPDFAChar}{{\rm MCPDFA-Char}\xspace}
\newcommand{\MCDFA}{{\rm MCDFA}\xspace}

\newcommand{\Pref}{{\rm Pref}\xspace}

\newcommand{\PTree}{{\rm PT}\xspace}

\newcommand{\SCCs}{{\rm SCCs}\xspace}
\newcommand{\Merged}{{\rm Merged}\xspace}

\newcommand{\Yes}{{\rm Yes}\xspace}
\newcommand{\No}{{\rm No}\xspace}
\newcommand{\True}{{\rm True}\xspace}
\newcommand{\False}{{\rm False}\xspace}

\usepackage{complexity}

\renewcommand{\NP}{{\rm NP}\xspace}

\newcommand*{\probleminternal}[4]{
	\par
	\medskip
	\noindent\fbox{\parbox{0.98\columnwidth}{
			\textbf{#4: #1} \\[0.05in]
			\renewcommand{\tabcolsep}{2pt}
			\begin{tabularx}{\linewidth}{rX}
				\emph{Input:} & #2 \\
				\emph{Output:} & #3
			\end{tabularx}
		}}
		\par
		\medskip
		\par
	}
	
\newcommand*{\problembox}[3]{\probleminternal{#1}{#2}{#3}{Problem}}
\newcommand*{\decproblembox}[3]{\probleminternal{#1}{#2}{#3}{Decision Problem}}

\acrodef{ltlf}[LTL$_f$]{Linear temporal logic over finite words}
\acrodef{pltlf}[PLTLf]{Preference over linear temporal logic over finite words}

%% file: ex-pdfa-learn/ex-a.tex
\tikzset{estate/.style={draw,ellipse,minimum width=1cm,minimum height=0.5cm}}
\begin{tikzpicture}[shorten >=1pt,node distance=1.5cm, on grid, auto, transform shape, scale=0.7]

  \node[estate,    initial] (eps) {$\varepsilon$};
  \node[estate, fill=orange!20,  above right=of eps] (a) {$a$};
    \node[estate,      right=of a] (ab) {$ab$};
      \node[estate,fill=orange!20,  above right=of ab] (aba) {$aba$};
      \node[estate,  fill=orange!20,   right=of ab] (abb) {$abb$};
    \node[estate,fill=blue!20,   above right=of a] (aa) {$aa$};
  \node[estate,   above right=of aa] (aab) {$aab$};
  \node[estate,  fill=blue!20,   right=of aab] (aabb) {$aabb$};
  \node[estate,   below right =of eps] (b) {$b$};
  \node[estate,  fill=darkgreen!20, above right=of b] (ba) {$ba$};  \node[estate,    fill=orange!20, right=of ba] (baa) {$baa$};
  \node[estate,  fill=blue!20,  right=of b] (bb) {$bb$};
  \node[estate,   fill=orange!20,  right=of bb] (bbb) {$bbb$};
  \node[estate,   fill=darkgreen!20, right=of aba] (abaa) {$abaa$};
  \node[estate,  fill=darkgreen!20,  right=of abb] (abbb) {$abbb$};

  \path[->]
    (eps) edge node {a} (a)
          edge node {b} (b)
    (a) edge node {a} (aa)
         edge node {b} (ab)
    (b) edge node {a} (ba)
         edge node {b} (bb)
    (aa) edge node {b} (aab)
    (ab) edge node {a} (aba)
         edge node {b} (abb)
    (ba) edge node {a} (baa)
    (bb) edge node {b} (bbb)
    (aab) edge node {b} (aabb)
    (abb) edge node {b} (abbb)
    (aba) edge node {a} (abaa);

\end{tikzpicture}

%% file: ex-pdfa-learn/ex-b.tex
\tikzset{estate/.style={draw,ellipse,minimum width=1cm,minimum height=0.5cm}}
\begin{tikzpicture}[shorten >=1pt,node distance=1.5cm, on grid, auto, transform shape, scale=0.7]

  \node[estate, fill=orange!20,   initial] (eps) {$\varepsilon$}; 
    \node[estate,      right=of a] (ab) {$ab$};
      \node[estate,fill=orange!20,  above right=of ab] (aba) {$aba$};
      \node[estate,  fill=orange!20,   right=of ab] (abb) {$abb$};
    \node[estate,fill=blue!20,   above right=of a] (aa) {$aa$};
  \node[estate,   above right=of aa] (aab) {$aab$};
  \node[estate,  fill=blue!20,   right=of aab] (aabb) {$aabb$};
  \node[estate,   below right =of eps] (b) {$b$};
  \node[estate,  fill=darkgreen!20, above right=of b] (ba) {$ba$};  \node[estate,    fill=orange!20, right=of ba] (baa) {$baa$};
  \node[estate,  fill=blue!20,  right=of b] (bb) {$bb$};
  \node[estate,   fill=orange!20,  right=of bb] (bbb) {$bbb$};
  \node[estate,   fill=darkgreen!20, right=of aba] (abaa) {$abaa$};
  \node[estate,  fill=darkgreen!20,  right=of abb] (abbb) {$abbb$};

  \path[->]
    (eps) edge [loop above] node {a} (eps)
          edge node {b} (b)
    (eps) edge node {a} (aa)
         edge node {b} (ab)
    (b) edge node {a} (ba)
         edge node {b} (bb)
    (aa) edge node {b} (aab)
    (ab) edge node {a} (aba)
         edge node {b} (abb)
    (ba) edge node {a} (baa)
    (bb) edge node {b} (bbb)
    (aab) edge node {b} (aabb)
    (abb) edge node {b} (abbb)
    (aba) edge node {a} (abaa);

\end{tikzpicture}

%% file: ex-pdfa-learn/ex-c.tex
\begin{tikzpicture}[shorten >=1pt,node distance=1.5cm, on grid, auto, transform shape, scale=0.7]

  \node[estate,    initial] (eps) {$\varepsilon$};
  \node[estate, fill=orange!20,  above right=of eps] (a) {$a$};
    \node[estate,      right=of a] (ab) {$ab$};
      \node[estate,fill=orange!20,  above right=of ab] (aba) {$aba$};
      \node[estate,  fill=orange!20,   right=of ab] (abb) {$abb$};
    \node[estate,fill=blue!20,   above right=of a] (aa) {$aa$};
  \node[estate,   above right=of aa] (aab) {$aab$};
  \node[estate,  fill=blue!20,   right=of aab] (aabb) {$aabb$};
   \node[estate,  fill=darkgreen!20, above right=of b] (ba) {$ba$};  \node[estate,    fill=orange!20, right=of ba] (baa) {$baa$};
  \node[estate,  fill=blue!20,  right=of b] (bb) {$bb$};
  \node[estate,   fill=orange!20,  right=of bb] (bbb) {$bbb$};
  \node[estate,   fill=darkgreen!20, right=of aba] (abaa) {$abaa$};
  \node[estate,  fill=darkgreen!20,  right=of abb] (abbb) {$abbb$};

  \path[->]
    (eps) edge node {a} (a)
            edge [loop above] node {b} (eps)
    (a) edge node {a} (aa)
         edge node {b} (ab)
    (eps) edge node {a} (ba)
         edge node {b} (bb)
    (aa) edge node {b} (aab)
    (ab) edge node {a} (aba)
         edge node {b} (abb)
    (ba) edge node {a} (baa)
    (bb) edge node {b} (bbb)
    (aab) edge node {b} (aabb)
    (abb) edge node {b} (abbb)
    (aba) edge node {a} (abaa);

\end{tikzpicture}

%% file: ex-pdfa-learn/ex-d.tex
\begin{tikzpicture}[shorten >=1pt,node distance=1.5cm, on grid, auto, transform shape, scale=0.7]

  \node[estate,    initial] (eps) {$\varepsilon$};
  \node[estate, fill=orange!20,    right=of eps] (a) {$a$};
    \node[estate,      right=of a] (ab) {$ab$};
      \node[estate,fill=orange!20,  above right=of ab] (aba) {$aba$};
      \node[estate,  fill=orange!20,   right=of ab] (abb) {$abb$};
    \node[estate,fill=blue!20,   above right=of a] (aa) {$aa$};
  \node[estate,   above right=of aa] (aab) {$aab$};
  \node[estate,  fill=blue!20,   right=of aab] (aabb) {$aabb$};
   \node[estate,  fill=darkgreen!20, below right=of a] (ba) {$ba$};  \node[estate,    fill=orange!20, right=of ba] (baa) {$baa$};
  \node[estate,  fill=blue!20,  below=of ba] (bb) {$bb$};
  \node[estate,   fill=orange!20,  right=of bb] (bbb) {$bbb$};
  \node[estate,   fill=darkgreen!20, right=of aba] (abaa) {$abaa$};
  \node[estate,  fill=darkgreen!20,  right=of abb] (abbb) {$abbb$};

  \path[->]
    (eps) edge node {a,b} (a)
          
    (a) edge node {a} (aa)
         edge node {b} (ab)
      edge node {a} (ba)
         edge [bend right] node [right] {b} (bb)
    (aa) edge node {b} (aab)
    (ab) edge node {a} (aba)
         edge node {b} (abb)
    (ba) edge node {a} (baa)
    (bb) edge node {b} (bbb)
    (aab) edge node {b} (aabb)
    (abb) edge node {b} (abbb)
    (aba) edge node {a} (abaa);

\end{tikzpicture}

%% file: ex-pdfa-learn/ex-e.tex
\tikzset{estate/.style={draw,ellipse,minimum width=1cm,minimum height=0.5cm}}
\begin{tikzpicture}[shorten >=1pt,node distance=1.5cm, on grid, auto, transform shape, scale=0.7]
  \node[estate, fill=blue!20,    initial] (eps) {$\varepsilon$};
  \node[estate, fill=orange!20,  above right=of eps] (a) {$a$};
    \node[estate,      right=of a] (ab) {$ab$};
      \node[estate,fill=orange!20,  above right=of ab] (aba) {$aba$};
      \node[estate,  fill=orange!20,   right=of ab] (abb) {$abb$};
   \node[estate,   above right =of a] (aab) {$aab$};
  \node[estate,  fill=blue!20,  above right=of aab] (aabb) {$aabb$};
  \node[estate,   below right =of eps] (b) {$b$};
  \node[estate,  fill=darkgreen!20, above right=of b] (ba) {$ba$};  \node[estate,    fill=orange!20, right=of ba] (baa) {$baa$};
  \node[estate,  fill=blue!20,  right=of b] (bb) {$bb$};
  \node[estate,   fill=orange!20,  right=of bb] (bbb) {$bbb$};
  \node[estate,   fill=darkgreen!20, right=of aba] (abaa) {$abaa$};
  \node[estate,  fill=darkgreen!20,  right=of abb] (abbb) {$abbb$};

  \path[->]
    (eps) edge  node {a} (a)
          edge node {b} (b)
           edge[bend left =120] node {b} (aab)
    (a) edge [bend left] node {a} (eps)
         edge node {b} (ab)
    (b) edge node {a} (ba)
         edge node {b} (bb)
    (ab) edge node {a} (aba)
         edge node {b} (abb)
    (ba) edge node {a} (baa)
    (bb) edge node {b} (bbb)
    (aab) edge node {b} (aabb)
    (abb) edge node {b} (abbb)
    (aba) edge node {a} (abaa);

\end{tikzpicture}

%% file: ex-pdfa-learn/ex-f.tex
\tikzset{estate/.style={draw,ellipse,minimum width=1cm,minimum height=0.5cm}}
\begin{tikzpicture}[shorten >=1pt,node distance=1.5cm, on grid, auto, transform shape, scale=0.7]
  \node[estate, fill=blue!20,    initial] (eps) {$\varepsilon$};
  \node[estate, fill=orange!20,  above right=of eps] (a) {$a$};
    \node[estate,      right=of a] (ab) {$ab$};
      \node[estate,fill=orange!20,  above right=of ab] (aba) {$aba$};
    \node[  above right =of a] (aab) { };
          \node[estate,  fill=orange!20,   right=of ab] (abb) {$abb$};

  \node[estate,  fill=blue!20,  above right=of aab] (aabb) {$aabb$};
  \node[estate,   below right =of eps] (b) {$b$};
  \node[estate,  fill=darkgreen!20, above right=of b] (ba) {$ba$};  \node[estate,    fill=orange!20, right=of ba] (baa) {$baa$};
  \node[estate,  fill=blue!20,  right=of b] (bb) {$bb$};
  \node[estate,   fill=orange!20,  right=of bb] (bbb) {$bbb$};
  \node[estate,   fill=darkgreen!20, right=of aba] (abaa) {$abaa$};
  \node[estate,  fill=darkgreen!20,  right=of abb] (abbb) {$abbb$};

  \path[->]
    (eps) edge  node {a} (a)
          edge node {b} (b)
     (a) edge [bend left] node {a} (eps)
         edge node {b} (ab)
    (b) edge node {a} (ba)
         edge node {b} (bb)
         edge [bend right=240] node {b} (aabb)
    (ab) edge node {a} (aba)
         edge node {b} (abb)
    (ba) edge node {a} (baa)
    (bb) edge node {b} (bbb)
   
    (abb) edge node {b} (abbb)
    (aba) edge node {a} (abaa);

\end{tikzpicture}

%% file: ex-pdfa-learn/ex-g.tex
\tikzset{estate/.style={draw,ellipse,minimum width=1cm,minimum height=0.5cm}}
\begin{tikzpicture}[shorten >=1pt,node distance=1.5cm, on grid, auto, transform shape, scale=0.7]
  \node[estate, fill=blue!20,    initial] (eps) {$\varepsilon$};
  \node[estate, fill=orange!20,  above right=of eps] (a) {$a$};
    \node[estate,      right=of a] (ab) {$ab$};
      \node[estate,fill=orange!20,  above right=of ab] (aba) {$aba$};
    \node[  above right =of a] (aab) { };
          \node[estate,  fill=orange!20,   right=of ab] (abb) {$abb$};

   \node[estate,    below right =of eps] (b) {$b$};
  \node[estate,  fill=darkgreen!20, above right=of b] (ba) {$ba$};  \node[estate,    fill=orange!20, right=of ba] (baa) {$baa$};
  \node[estate,  fill=blue!20,  right=of b] (bb) {$bb$};
  \node[estate,   fill=orange!20,  right=of bb] (bbb) {$bbb$};
  \node[estate,   fill=darkgreen!20, right=of aba] (abaa) {$abaa$};
  \node[estate,  fill=darkgreen!20,  right=of abb] (abbb) {$abbb$};

  \path[->]
    (eps) edge  node {a} (a)
          edge node {b} (b)
     (a) edge [bend left] node {a} (eps)
         edge node {b} (ab)
    (b) edge node {a} (ba)
         edge node {b} (bb)
     (ab) edge node {a} (aba)
         edge node {b} (abb)
    (ba) edge node {a} (baa)
    (bb) edge node {b} (bbb)
   
    (abb) edge node {b} (abbb)
    (aba) edge node {a} (abaa);

\end{tikzpicture}

%% file: ex-pdfa-learn/ex-h.tex
\tikzset{estate/.style={draw,ellipse,minimum width=1cm,minimum height=0.5cm}}
\begin{tikzpicture}[shorten >=1pt,node distance=1.5cm, on grid, auto, transform shape, scale=0.7]
  \node[estate, fill=blue!20,    initial] (eps) {$\varepsilon$};
  \node[estate, fill=orange!20,  above right=of eps] (a) {$a$};
    \node[estate,  fill=darkgreen!20,  below  right=of a] (ab) {$ab$};
      \node[estate,fill=orange!20,  above right=of ab] (aba) {$aba$};
    \node[  above right =of a] (aab) { };
          \node[estate,  fill=orange!20,   right=of ab] (abb) {$abb$};

   \node[estate,    below right =of eps] (b) {$b$};
  
  \node[estate,  fill=blue!20,  right=of b] (bb) {$bb$};
  \node[estate,   fill=orange!20,  right=of bb] (bbb) {$bbb$};
  \node[estate,   fill=darkgreen!20, right=of aba] (abaa) {$abaa$};
  \node[estate,  fill=darkgreen!20,  right=of abb] (abbb) {$abbb$};

  \path[->]
    (eps) edge  node {a} (a)
          edge node {b} (b)
     (a) edge [bend left] node {a} (eps)
         edge node {b} (ab)
    (b) edge node {a} (ab)
         edge node {b} (bb)
     (ab) edge node {a} (aba)
         edge node {b} (abb)
    (bb) edge node {b} (bbb)
   
    (abb) edge node {b} (abbb)
    (aba) edge node {a} (abaa);

\end{tikzpicture}

%% file: ex-pdfa-learn/ex-i.tex
\tikzset{estate/.style={draw,ellipse,minimum width=1cm,minimum height=0.5cm}}
\begin{tikzpicture}[shorten >=1pt,node distance=1.5cm, on grid, auto, transform shape, scale=0.7]
  \node[estate, fill=blue!20,    initial] (eps) {$\varepsilon$};
  \node[estate, fill=orange!20,  above right=of eps] (a) {$a$};
    \node[estate,  fill=darkgreen!20,   below right=of a] (ab) {$ab$};
      \node[estate,fill=orange!20,  above right=of ab] (aba) {$aba$};
    \node[  above right =of a] (aab) { };
          \node[estate,  fill=orange!20,   right=of ab] (abb) {$abb$};

   \node[estate,    fill=orange!20,   below right =of eps] (b) {$b$};
  
  \node[estate,   fill=darkgreen!20, right=of aba] (abaa) {$abaa$};
  \node[estate,  fill=darkgreen!20,  right=of abb] (abbb) {$abbb$};

  \path[->]
    (eps) edge  node {a} (a)
          edge node {b} (b)
     (a) edge [bend left] node {a} (eps)
         edge node {b} (ab)
    (b) edge node {a} (ab)
         edge [bend left] node {b} (eps)
     (ab) edge node {a} (aba)
         edge node {b} (abb)
   
    (abb) edge node {b} (abbb)
    (aba) edge node {a} (abaa);

\end{tikzpicture}

%% file: ex-pdfa-learn/ex-j.tex
\tikzset{estate/.style={draw,ellipse,minimum width=1cm,minimum height=0.5cm}}
\begin{tikzpicture}[shorten >=1pt,node distance=1.5cm, on grid, auto, transform shape, scale=0.7]
  \node[estate, fill=blue!20,    initial] (eps) {$\varepsilon$};
  \node[estate, fill=orange!20,  above right=of eps] (a) {$a$};
    \node[estate,  fill=darkgreen!20,   below right=of a] (ab) {$ab$};
    \node[  above right =of a] (aab) { };
        \node[estate,  fill=orange!20,   right=of ab] (abb) {$abb$};

   \node[estate,    fill=orange!20,   below right =of eps] (b) {$b$};
  
   \node[estate,  fill=darkgreen!20,  right=of abb] (abbb) {$abbb$};

  \path[->]
    (eps) edge  node {a} (a)
          edge node {b} (b)
     (a) edge [bend left] node {a} (eps)
         edge node {b} (ab)
    (b) edge node {a} (ab)
         edge [bend left] node {b} (eps)
     (ab)   edge [bend left] node {a} (b)
          edge node {b} (abb)
   
   (abb) edge node {b} (abbb);

\end{tikzpicture}

%% file: ex-pdfa-learn/ex-k.tex
\tikzset{estate/.style={draw,ellipse,minimum width=1cm,minimum height=0.5cm}}
\begin{tikzpicture}[shorten >=1pt,node distance=1.5cm, on grid, auto, transform shape, scale=0.7]
  \node[estate, fill=blue!20,    initial] (eps) {$\varepsilon$};
  \node[estate, fill=orange!20,  above right=of eps] (a) {$a$};
    \node[estate,  fill=darkgreen!20,   below right=of a] (ab) {$ab$};
    \node[  above right =of a] (aab) { };

   \node[estate,    fill=orange!20,   below right =of eps] (b) {$b$};
  

  \path[->]
    (eps) edge  node {a} (a)
          edge node {b} (b)
     (a) edge [bend left] node {a} (eps)
         edge [bend left] node {b} (ab)
    (b) edge node {a} (ab)
         edge [bend left] node {b} (eps)
     (ab)   edge [bend left] node {a} (b)
          
          edge   node {b} (a);
   

\end{tikzpicture}

%% file: figs/running-example-pdfa.tex
\begin{tikzpicture}[->, >=stealth, shorten >=1pt, auto, node distance=2cm, semithick, draw=black]

  \node[estate,   initial] (s0) [fill=blue!20,draw] {$s_0$: $2$  };
  \node[estate] (s1) [below left of=s0,fill= orange!20,draw] {$s_1$: $1$  };
  \node[estate] (s2) [below right of=s0, fill= orange!20,draw] {$s_2$: $1$  };
  \node[estate] (s3) [below right of=s1,fill=darkgreen!20,draw] {$s_3$: $0$};

  \path (s0) edge  node {a} (s1)
        (s1) edge [bend left]  node {a} (s0)
        (s0) edge  node [right]   {b} (s2)
        (s2) edge [bend right] node [right] {b} (s0)
        (s1) edge node {b} (s3)
        (s3) edge [bend left] node {b} (s1)
        (s3) edge node {a} (s2)
        (s2) edge [bend left] node {a} (s3);

\end{tikzpicture}

%% file: figs/running-example-prefgraph.tex
\begin{tikzpicture}[->, >=stealth, shorten >=1pt, node distance=1cm, semithick]

  \node[draw, circle, fill=orange!20] (1) {1};
  \node[draw, circle, above left= of 1,fill=darkgreen!20] (0) {0};
  \node[draw, circle, above right=  of 1,fill=blue!20] (2) {2};

  \draw (1) -- (0);
  \draw (1) -- (2);

\end{tikzpicture}